\documentclass[sigconf]{acmart}
\setcopyright{acmcopyright}
\copyrightyear{2024}
\acmYear{2024}
\acmDOI{XXXXXXX.XXXXXXX}

\acmConference[WSDM '24]{The 17th ACM International Conference on Web Search and Data Mining}{March 04--08,
  2024}{Merida, Mexico}
\acmBooktitle{WSDM '24: The 17th ACM International Conference on Web Search and Data Mining, March 04--08,
  2024, Merida, Mexico} 
\acmPrice{15.00}
\acmISBN{978-1-4503-XXXX-X/18/06}

\usepackage{xcolor}
\usepackage[ruled,linesnumbered]{algorithm2e}
\usepackage{graphicx}
\usepackage{subfig}
\usepackage{threeparttable}
\usepackage{booktabs}
\usepackage{multirow}
\usepackage{caption}
\usepackage{url}
\usepackage{diagbox}
\usepackage{svg}
\usepackage{enumitem}
\newcommand{\cheng}{\color{red}}

\begin{document}
\title{Defense Against Model Extraction Attacks on Recommender Systems}

\author{Sixiao Zhang}
\email{sixiao001@e.ntu.edu.sg}
\affiliation{%
  \institution{Nanyang Technological University}
  \country{Singapore}
}

\author{Hongzhi Yin}
\authornote{Co-corresponding authors.}
\email{h.yin1@uq.edu.au}
\affiliation{%
  \institution{The University of Queensland}
  \city{Brisbane}
  \country{Australia}
}

\author{Hongxu Chen}
\email{hongxu.chen@uq.edu.au}
\affiliation{%
  \institution{The University of Queensland}
  \city{Brisbane}
  \country{Australia}}

\author{Cheng Long}
\authornotemark[1]
\email{c.long@ntu.edu.sg}
\affiliation{%
  \institution{Nanyang Technological University}
  \country{Singapore}
}

\renewcommand{\shortauthors}{Zhang, et al.}

\begin{abstract}

The robustness of recommender systems has become a prominent topic within the research community. Numerous adversarial attacks have been proposed, but most of them rely on extensive prior knowledge, such as all the white-box attacks or most of the black-box attacks which assume that certain external knowledge is available. Among these attacks, the model extraction attack stands out as a promising and practical method, involving training a surrogate model by repeatedly querying the target model. However, there is a significant gap in the existing literature when it comes to defending against model extraction attacks on recommender systems. In this paper, we introduce Gradient-based Ranking Optimization (GRO), which is the first defense strategy designed to counter such attacks. We formalize the defense as an optimization problem, aiming to minimize the loss of the protected target model while maximizing the loss of the attacker's surrogate model. Since top-k ranking lists are non-differentiable, we transform them into swap matrices which are instead differentiable. These swap matrices serve as input to a student model that emulates the surrogate model's behavior. By back-propagating the loss of the student model, we obtain gradients for the swap matrices. These gradients are used to compute a swap loss, which maximizes the loss of the student model. We conducted experiments on three benchmark datasets to evaluate the performance of GRO, and the results demonstrate its superior effectiveness in defending against model extraction attacks.
\end{abstract}



\maketitle
\pagestyle{plain}
\sloppy
\section{Introduction}
Recommender systems, which provide suggestions for items that best fit the user's preference, are ubiquitous in our daily lives. They serve as important components in e-commerce \cite{schafer2001commerce}, social platforms \cite{fan2019graph}, healthcare \cite{yue2021overview}, finance \cite{sharaf2022survey}, and more. A good recommender system is vital for both users and service providers, as it significantly improves user experience by directing them to new items or products that precisely match their preferences. This, in turn, increases the number of active users and leads to higher profits for the service providers.

Service providers integrate recommender systems into their products and make them accessible to the target users or the public. However, two significant problems need to be addressed before deploying recommender systems: robustness and information leakage \cite{fan2022comprehensive, ge2022survey, wang2021fast, zhang2022pipattack, zhang2023comprehensive}. On the one hand, recommender systems are sensitive to noise in the training data, where even a small perturbation can lead to a significant degradation in performance \cite{fang2018poisoning, yuan2023manipulating}. On the other hand, recommender systems often involve intellectual property that needs protection, or contain private information of the training data that should not be disclosed \cite{zhang2021membership, zhang2021graph, yuan2023interaction}. Therefore, it is crucial to find ways to protect recommender systems against various adversarial attacks that aim to either poison the model or extract specific information.

Most existing adversarial attack methods for recommender systems assume that the attacker has certain prior knowledge. For example, some attacks are white-box or gray-box attacks \cite{lin2020attacking,li2016data,yang2017fake}, where the attacker has access to the target recommender system or the data. Other attacks are black-box attacks but assume that the attacker has access to external knowledge, such as data from a different domain \cite{fan2021attacking}, item metadata \cite{chen2022knowledge}, or some of the real users from the target data \cite{zeng2023practical}. However, such assumptions are often invalid in real-world scenarios, as the required knowledge is not always available. On the one hand, the attacker can hardly obtain access to the target model since service providers rarely publish their models. On the other hand, some works assume that the user-item interaction data is visible on the corresponding platforms, which can be easily collected by the attacker. In fact, most platforms nowadays do not make the user information public. Therefore, it is impractical to make such an assumption. For example, the attacker cannot see other users' interaction history on most video platforms and news platforms, such as YouTube, TikTok, BBC news, etc. These platforms either provide an option for users to make their personal information private, or do not allow to check other people's information at all. Additionally, for online shopping platforms like Amazon and Shopee, the attacker cannot see a person's full interaction history, as well as seeing all the people that have bought a certain item. Even if the attacker can identify a certain customer who has written a review for a product, the attacker has no idea what other products the customer has purchased. Although the attacker can obtain a dataset by crawling the target platform greedily, it is an extremely time-consuming process, and the quality of such a dataset is not guaranteed.


To attack a target model without prior knowledge, one black-box attack method called \emph{model extraction/stealing attack} \cite{oliynyk2022know, gong2020model, chakraborty2018adversarial} has been proposed. The attacker repeatedly queries the target model with synthesized data, and uses the feedback of the target model as labels to create a local dataset. This local dataset is used to train a local surrogate model to approximate the performance of the target model. By treating the surrogate model as the replacement of the target model, various tasks can be performed, including the adversarial attacks. Attacks are carried out based on the surrogate model and then transferred to the target model. This model extraction attack is powerful because it can be used to attack any machine learning models without any prior knowledge. There are only two universally effective defense methods, both with high risks. One method is to detect suspicious queries, but it can be easily bypassed by changing the query patterns. The other method is to alter the output of the model to fool the surrogate model, but this also decreases the utility of the target model.

Researchers have studied how to defend against model extraction attacks in classification tasks \cite{juuti2019prada,kariyappa2020defending,lee2019defending,orekondy2020prediction}. One line of research is to detect out-of-distribution queries \cite{kariyappa2020defending,juuti2019prada}. However, for recommender systems, out-of-distribution queries are hard to define for a sequence of user interactions. Another line of research is to change the model's output by treating the defense as an optimization problem \cite{lee2019defending,orekondy2020prediction}. However, in classification tasks, the predicted class probabilities are assumed to be visible to the attacker, so the defense methods focus on changing the probabilities of each class while maintaining accuracy, ensuring the probability of the true class is always the largest. Such methods cannot be directly applied to recommender systems, as recommender systems provide top-k rankings instead of class probabilities. Some works have also attempted to do model watermarking \cite{boenisch2020survey,shafieinejad2021robustness,xu2021watermarking}, where the surrogate model always produces similar outputs to the target model for certain queries. However, watermarking cannot prevent the model from being stolen, so it is not a primary choice in real-world scenarios.

To the best of our knowledge, there are no existing defense methods against model extraction attacks on recommender systems. In this paper, we propose a defense method called Gradient-based Ranking Optimization (GRO). The basic idea of GRO is to learn a target model whose output will maximize the loss of the attacker's surrogate model. Specifically, we use a student model as a replacement of the attacker's surrogate model. The student model tries to extract the target model by training on the top-k lists generated by the target model. We calculate the gradients of the top-k lists w.r.t. the loss of the student model. We can infer how to increase the loss of the student model by altering the top-k lists according to the gradients. However, we cannot directly obtain the gradients of the top-k lists because they are discrete and non-differentiable. We instead convert the list into a swap matrix $\mathbf{A}$, where $\mathbf{A}_{ij}=1$ if the $j$-th item is ranked at the $i$-th position. The gradients of such swap matrices can be easily obtained. We use these gradients to define a new swap matrix $\mathbf{A}'$ by setting the entries with the largest positive gradients to 1. If $\mathbf{A}'$ is used as the input to the student model, the loss of the student model will be maximized. Therefore, we define a swap loss to force $\mathbf{A}$ to approximate $\mathbf{A}'$, so that the target model will learn to fool the student model. 
The target model trained with GRO can be deployed directly. The black-box model extraction attacks will acquire a surrogate model with much worse performance than the target model. We make our implementation of GRO available online\footnote{\url{https://github.com/RinneSz/GRO-Gradient-based-Ranking-Optimization}}. Our contributions are summarized as follows:
\begin{itemize}
    \item We propose GRO, which is, to our knowledge, the first defense method against model extraction attacks on recommender systems. GRO is a general framework and can be used to protect any recommender systems.
    \item We propose to compute the gradients of the top-k ranking lists by converting them into swap matrices. Such gradients are used to compute a swap loss which can maximize the loss of the attacker's surrogate model.
    \item Extensive experiments show that GRO can effectively protect the target model from model extraction attacks by reducing the performance of the attacker's model while maintaining the utility of the target model.
\end{itemize}
\section{related work}

\subsection{Black-box Adversarial Attacks on Recommender Systems}
The robustness of machine learning models has been studied extensively by the research community \cite{szegedy2013intriguing,goodfellow2014explaining,papernot2016limitations,samangouei2018defense}. It is one of the major challenges when deploying machine learning models to real-world applications, such as recommender systems. Plenty of works have been proposed to explore the robustness of recommender systems \cite{yang2017fake,huang2021data,lin2020attacking,fan2021attacking,song2020poisonrec,chen2022knowledge,lin2022shilling,zeng2023practical}. Many of them assume a black-box setting, where the attacker knows nothing about the target model and its training data. But most of them require external knowledge or extra operations. For example, Fan et al. \cite{fan2021attacking} proposed CopyAttack, which uses reinforcement learning to attack black-box recommender systems. However, they require access to user data of a different domain. For example, to attack Amazon's model, they use the user profiles in eBay. They pre-define some spy users in the target model. The reward of the reinforcement learning is computed by observing the recommendations made to these spy users. Song et al. \cite{song2020poisonrec} proposed PoisonRec, a reinforcement learning method to attack black-box recommender systems. They use the RecNum as the reward, which is the number of times the item is recommended (clicked) by all users. Chen et al. \cite{chen2022knowledge} proposed KGAttack to attack black-box recommender systems with the help of the item knowledge graph (KG). The item KG is obtained by assuming that the metadata of the items is accessible. They first use TransE \cite{bordes2013translating} and GCN \cite{kipf2016semi} to get knowledge-enhanced item embeddings from the KG, then use such embeddings to train a reinforcement learning model to learn to inject fake user profiles. The reward is computed by spy users, which is the same as CopyAttack. Lin et al. \cite{lin2022shilling} proposed Leg-UP, a reinforcement learning black-box attack method. However, it requires access to some of the real users to use a GAN to generate fake users that are similar to benign users. Zeng et al. \cite{zeng2023practical} proposed PC-Attack. It requires access to data from a different domain, similar to that of CopyAttack. It also requires partial data from the target domain. Graph topology is captured by training with contrastive learning on the cross-domain data. Data from the target domain is used to fine-tune the model to generate fake users.

The above black-box adversarial attacks are impractical to be applied to real-world scenarios. They require either external knowledge \cite{zeng2023practical} or retraining the target model \cite{song2020poisonrec}, or both \cite{fan2021attacking,chen2022knowledge,lin2022shilling}. On the one hand, there is no guarantee that high-quality external knowledge is always available. On the other hand, retraining the target model is simply impossible for the attacker. Compared to these attacks, the model extraction attack is a more practical black-box attack, where a surrogate model is trained on the queries and outputs of the target model \cite{zhang2021reverse,yue2021black}. This surrogate model is used as a replacement for the target model to perform various downstream attacks, such as profile pollution attacks and data poisoning attacks. The model extraction attack is a very strong and universal attack since almost all the adversarial attacks, no matter white-box or black-box, can be conducted using the surrogate model and then transferred to the target model. Zhang et al. \cite{zhang2021reverse} proposed Reverse Attack. They train a surrogate model to approximate the target model by training on observed ranking lists. These ranking lists are crawled on the websites of those platforms. No user information is included. Both the training and inference use the similarity of item embeddings as the criteria. Yue et al. \cite{yue2021black} proposed a black-box model extraction attack on recommender systems. They query the target model and use the top-k ranking lists returned by the target model to train a surrogate model. The surrogate model is forced to return similar top-k ranking lists to those by the target model. This surrogate model is used as a replacement for the target model to perform various downstream attacks.

\subsection{Defenses Against Model Extraction Attacks in Other Domains}
People have tried to defend against model extraction attacks in other domains, most of which are for classification tasks \cite{lee2019defending,orekondy2020prediction,kariyappa2020defending,juuti2019prada}. However, they are not applicable to the recommender systems because their model assumptions and attack settings are different. For example, Lee et al. \cite{lee2019defending} proposed Reverse Sigmoid, an activation function used at the last layer of neural networks to replace the traditional Sigmoid. However, one of the premises of this defense is that the attacker can access the output posterior class probabilities. Orekondy et al. \cite{orekondy2020prediction} modeled the defense problem as a bi-level optimization problem. The goal of the optimization problem is to maximize the gradient deviation between the target model and the surrogate model, to mislead the surrogate model into a different gradient direction. It also assumes that the attacker can access the class probabilities. Kariyappa et al. \cite{kariyappa2020defending} proposed detecting suspicious queries which are out of distribution (OOD). Then, the output of those queries has minimized probabilities for the correct classes. Juuti et al. \cite{juuti2019prada} proposed PRADA, another detection method to identify OOD queries. However, for recommender systems, OOD queries are more difficult to detect. On the one hand, it is hard to define OOD patterns for sequences. On the other hand, real-world user behaviours are highly irregular, so it is difficult to distinguish abnormal queries from benign queries.






\section{Preliminaries}
\subsection{Problem Definition}
Model extraction attacks query the target model with their own data and then use the output of the target model to train a surrogate model. The surrogate model can have a different architecture from the target model, as both the attacker and the defender have no idea about each other's model architecture. This process can be formalized as follows:
\begin{equation}
\label{eq:target goal}
    \mathop{\text{min}}\limits_{\phi}\ L_{\text{surrogate}}(x,f_{\theta},g_{\phi})=|f_{\theta}(x)-g_{\phi}(x)|
\end{equation}
Here, $L_{\text{surrogate}}$ represents the loss of the surrogate model, $x$ is the query, $f$ is the target model, $\theta$ is the parameter of $f$, $g$ is the surrogate model, and $\phi$ is the parameter of $g$. The attacker's goal is to minimize the difference between the output of the target model and the output of the surrogate model.

To defend against model extraction attacks, we need to maximize the loss of the surrogate model while minimizing the loss of the target model. This can be formalized as:

\begin{gather}
    \label{eq:surrogate goal}
     \mathop{\text{min}}\limits_{\theta}\ L_{\text{target}}(x,y,f_{\theta})-L_{\text{surrogate}}(x,f_{\theta},g_{\phi}),\\
     \text{s.t.}\ \phi =\mathop{\text{argmin}}\limits_{\phi} L_{\text{surrogate}}(x,f_{\theta},g_{\phi}) 
\end{gather}

Here, $L_{\text{target}}$ is the original loss function of the target model, and $y$ is the label. This is a bi-level optimization problem. There are two major challenges in solving this problem. First, the defender has no access to the surrogate model, so they cannot optimize $L_{\text{surrogate}}$ directly. But this can be addressed by using a local model to simulate the surrogate model. Second, $f_{\theta}(x)$ is a discrete ranking list. We cannot obtain its gradient. This makes it impossible to perform back-propagation and optimization. However, we can convert it into a \emph{swap matrix} so that its gradient can be computed. We will discuss how we convert the ranking list into a swap matrix in \autoref{sec:ranking to swap}.

\begin{figure*}[t]
\centering
 \includegraphics[width=1.0\linewidth]{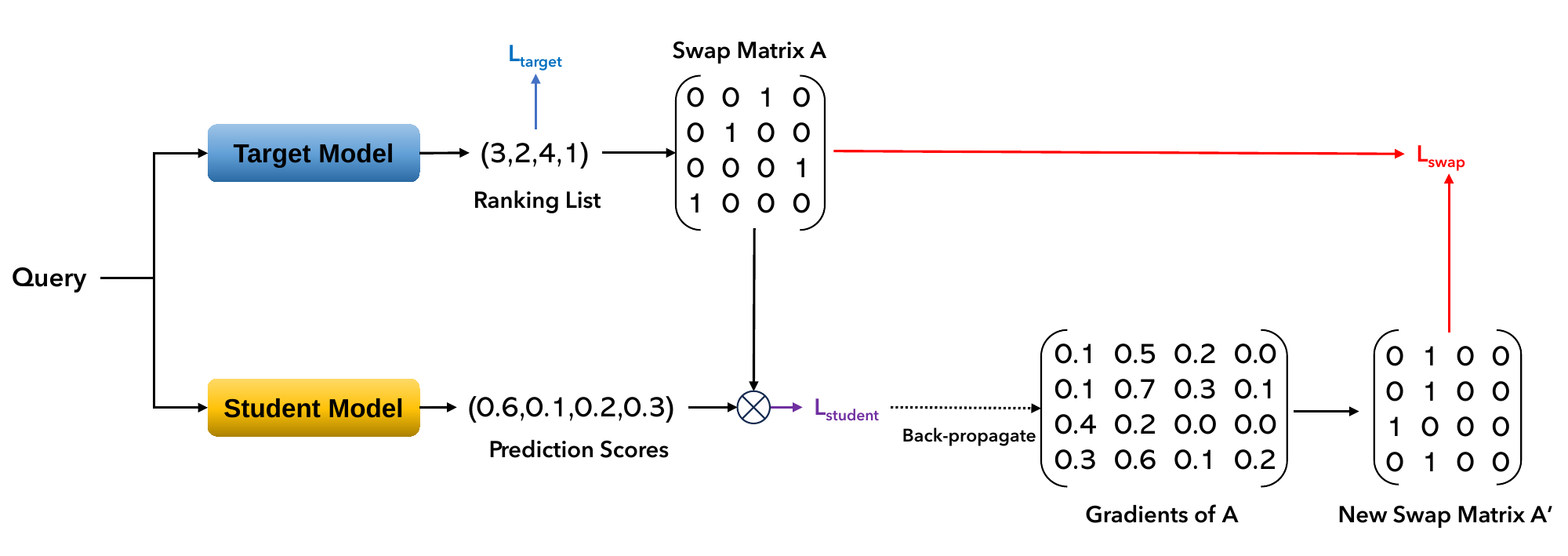}
\vspace{-3em}
\caption{The workflow of GRO. The target model produces a ranking list for the input query. The ranking list is converted into a swap matrix $\mathbf{A}$. Then $\mathbf{A}$ is multiplied with the output of the student model to calculate its loss. We back-propagate the loss and obtain the gradients of $\mathbf{A}$, which are then converted into a new swap matrix $\mathbf{A}'$. A swap loss is calculated using $\mathbf{A}$ and $\mathbf{A}'$ to learn a target model that can fool the student model.}
\label{fig:GRO flow}
\vspace{-1.3em}
\end{figure*}

\subsection{Attacker}
Before introducing our defense method, we first introduce how the model extraction attack against recommender systems works. The attacker begins by generating a fake user interaction history according to certain strategy, e.g. randomly or autoregressively \cite{yue2021black}. Then, the attacker queries the target model and obtains a top-k recommendation list. This list is ordered by preference scores, with items having high scores ranked at the top. The attacker uses these rankings to train a surrogate model that approximates the performance of the target model. The training objective of the attacker is to align the scores of the items with the ordering. It consists of two parts: first, pushing the item ranked higher to have a higher score than items ranked lower, and second, ensuring that the items in the top-k list have higher scores than items outside the top-k list. Specifically, the attacker's loss function is defined as follows:
\begin{equation}
\label{eq:surrogate loss}
    L_{\text{surrogate}}=\sum_{i=1}^{k-1}\text{max}(s_{i+1}-s_{i}+m_{1},0)+\sum_{i=1}^{k}\text{max}(s_{i}^{neg}-s_{i}+m_{2},0)
\end{equation}
Here, $s_{i}$ is the surrogate model's prediction score for the item ranked at the $i$-th position according to the target model's top-k list, $s_{i}^{neg}$ is a sampled negative item that ranks outside the top-k list, and $m_{1}$ and $m_{2}$ are two hyper-parameters indicating the margins. The function max($\cdot,\cdot$) returns the maximum value between its arguments. By training the surrogate model with this loss function, it learns to produce similar top-k lists to the target model.

\section{Gradient-based Ranking Optimization}
\subsection{Challenges}
To defend against model extraction attacks, an effective solution is to maximize the loss of the attacker's surrogate model. In a strict black-box setting, where the attacker can only use the top-k ranking lists generated by the target model to train the surrogate model, we aim to learn a target model whose generated top-k ranking lists will maximize the loss of the surrogate model. However, since the ranking lists are discrete, we cannot directly obtain the gradients of them. We need to find an approach to calculate those gradients. This is the main challenge in designing our method.

\subsection{Overview}
In this section, we introduce our Gradient-based Ranking Optimization (GRO). The basic idea of GRO is to maximize the loss of the surrogate model while minimizing the loss of the target model. GRO is implemented during the training phase of the target model. By training the target model with GRO, it becomes capable of deceiving model extraction attacks while preserving good utility. An illustration of GRO workflow is depicted in \autoref{fig:GRO flow}. A student model is used to mimic the behaviour of the attacker's surrogate model, whose goal is to extract the target model. In each iteration, the target model produces the ranking list of the input query, and then convert it into a swap matrix $\mathbf{A}$. The student model calculates its loss based on the swap matrix. We back-propagate the loss and obtain the gradients of the swap matrix, and convert the gradients into a new swap matrix $\mathbf{A}'$. A swap loss is calculated based on $\mathbf{A}$ and $\mathbf{A}'$. Thus, the target model can learn how to increase the loss of the student model by approximating its output to $\mathbf{A}'$. Next, we introduce each step in detail.

\subsection{The Target Model and The Student Model}
We first introduce the design of the target model and the student model. In model extraction attacks, it is essential for the target model to be capable of processing unseen queries (i.e., the target model should be inductive). Without loss of generality, we assume a sequential recommendation model as the target model. This is also a practical choice for real-world recommender systems, where encountering new and unseen queries is common. The target model can be any sequential model that takes a sequence of items as input and generates a top-k ranking list as the next item recommendation. The cross-entropy loss is usually used as the loss function in sequential recommendations. In our GRO, we assume that the student model has the same architecture as the target model, which is the worst case for defense where the attacker somehow figures out the target model's architecture. The input of the student model is the same sequence used to train the target model. To train the student model, we use the same loss function as \autoref{eq:surrogate loss}, supervised by the top-k ranking lists produced by the target model. In this way, the student model serves as an approximation of the attacker's surrogate model. We can increase the surrogate model's loss by increasing the student model's loss.

\subsection{Converting Top-k Ranking Lists to Swap Matrices}
\label{sec:ranking to swap}
To maximize the loss of the student model, it is necessary to obtain the gradients of the top-k ranking lists. This is non-trivial since the ranking lists are non-differentiable. To solve this problem, we convert the ranking lists into the swap matrices which are instead differentiable.

Assume that the dataset contains a total of $m$ items. For a given top-k ranking list, we can construct a swap matrix $\mathbf{A}$ with dimensions $k\times m$. Suppose the top-k ranking list is $(a_{1},a_{2},a_{3},...,a_{k})$, where $a_{i}$ is the ID (starting from 1) of the item ranked at the $i$-th position. In the swap matrix, the entry $\mathbf{A}_{i,a_{i}}$ is set to 1, while all other entries are set to 0. For instance, if the item with ID=3 is ranked at the first position, then $\mathbf{A}_{1,3}$ would be 1. An example of converting the ranking list into the swap matrix is shown in \autoref{fig:GRO flow} upper center. In such a swap matrix, each row contains only one non-zero entry, which indicates the specific item ranked at that position. By performing a matrix multiplication between the swap matrix and the item ID vector $(1,2,3,...,m)^\top$, we obtain the exact top-k ranking list produced by the target model.

Meanwhile, the student model produces a 1-D vector $\mathbf{S}_{\text{student}}$ of scores for all items, ordered by their IDs. This vector has a size of $m\times 1$. By doing a matrix multiplication between $\mathbf{A}$ and $\mathbf{S}_{\text{student}}$, we obtain a new list of scores with size $k\times 1$. This new list preserves the same ordering as the top-k ranking list generated by the target model. It can then be utilized to compute the loss function of the student model using \autoref{eq:surrogate loss}. We can back-propagate the loss and obtain the gradient of $\mathbf{A}$.

\subsection{Gradient Computation and Optimization}
By back-propagating the loss function of the student model, we can acquire the gradients of the swap matrix $\mathbf{A}$. Now we need to design an objective function based on the gradients to train the target model in order to maximize the loss of the student model.


Each entry in $\mathbf{A}$ has a corresponding gradient value. If the gradient is positive, it means that increasing the value of that entry can generally increase the loss. Instead, if the gradient is negative, it means that decreasing the value of that entry can generally increase the loss. Therefore, if an entry in $\mathbf{A}$ is 0, and its gradient is positive, we can change its value to 1 so that the loss is expected to increase. For each row, by setting the entry with the largest gradient to 1, and setting all other entries to 0, we can obtain a new swap matrix $\mathbf{A}'$. An example is shown in \autoref{fig:GRO flow} lower right. If the output of the target model is exactly $\mathbf{A}'$, then the loss of the student model will be maximized. However, simply forcing the model to learn to output $\mathbf{A}'$ is not a good choice. On the one hand, we want to preserve the utility of the target model, but $\mathbf{A}'$ can be very different from $\mathbf{A}$, which will significantly degrade the performance of the model. On the other hand, the model's output may not converge to $\mathbf{A}'$ since $\mathbf{A}'$ can be an invalid swap matrix. Note that one item may be ranked at multiple positions simultaneously in $\mathbf{A}'$. For example, in \autoref{fig:GRO flow}, the second item has the largest gradient in the first, the second, and the fourth row, then it will be ranked at the first, the second, and the fourth position at the same time, which is impossible for a valid swap matrix where each item can only be ranked at one position. Therefore, we cannot let the model to learn the raw $\mathbf{A}'$ directly. We need to design a loss function that can achieve two goals: first, it should let the output of the target model approach $\mathbf{A}'$ as much as possible, while still being a valid swap matrix; second, the learned swap matrix should preserve the original ranking as much as possible, to avoid a severe utility degradation. 

To this end, we define the swap loss as:
\begin{equation}
\label{eq:swap loss}
    L_{\text{swap}}=\frac{1}{k}\sum_{i=0}^{k-1}\text{max}((\mathbf{A}_{i}-\mathbf{A}'_{i})\mathbf{S}_{\text{target}}+m_{\text{swap}},0)
\end{equation}
where $k$ is the number of items in the output ranking list (top-k), $\mathbf{A}_{i}$ and $\mathbf{A}'_{i}$ are the $i$-th row of the original swap matrix and the new swap matrix, $\mathbf{S}_{\text{target}}$ is the target model's predicted scores for all items ordered by the item IDs, $m_{\text{swap}}$ is a hyper-parameter that denotes the margin, max($\cdot,\cdot$) returns the maximum value. Such a swap loss pushes the item with the largest gradient to have a higher score than the item ranked at the corresponding position. When converged, this loss can ensure that both the orderings of $\mathbf{A}$ and $\mathbf{A}'$ are preserved as much as possible. 

\begin{lemma}
\label{lemma:swap}
Suppose the top-k rankings indicated by $\mathbf{A}$ and $\mathbf{A}'$ are $(a_{1},a_{2},a_{3},...,a_{k})$ and $(a_{1}',a_{2}',a_{3}',...,a_{k}')$ respectively, where $a_{i}$ and $a_{i}'$ denote the items ranked at the $i$-th position. When the swap loss defined in \autoref{eq:swap loss} (assume $m_{\text{swap}}=0$) is converged to 0, for $\forall i\in[1,k]$, if $a_{i}'\neq a_{j}'$ for $\forall j\in[1,i)$, we have $a_{i}=a_{i}'$.
\end{lemma}
\begin{proof}
Let's consider the first two rows of $\mathbf{A}$ and $\mathbf{A}'$ for an example. Suppose the scores of item $a_{i}$ and $a_{i}'$ predicted by the target model are $s_{i}$ and $s_{i}'$ respectively. When the loss is converged to 0, for $a_{1}$ and $a_{1}'$, we have $s_{1}\leq s_{1}'$. If $s_{1}<s_{1}'$ is true, it means that $a_{1}\neq a_{1}'$, because they cannot be the same item if they have different scores. In other words, $a_{1}$ and $a_{1}'$ are two different items. Therefore, $a_{1}'$ should rank higher than $a_{1}$ in $\mathbf{A}$ because $a_{1}'$ has a higher score. However, $a_{1}$ is already ranked at the first position in $\mathbf{A}$, so it is contradictory. Therefore, we can only have $s_{1}=s_{1}'$. In this case, $a_{1}$ and $a_{1}'$ might be the same item, or might be different items with the same score. Without loss of generality, we assume that they are the same item, since they are interchangeable if they are different items. Then we consider $a_{2}$ and $a_{2}'$. If converged, we have $s_{2}\leq s_{2}'$. If $s_{2}=s_{2}'$ is true, then $a_{2}$ and $a_{2}'$ are the same item. If $s_{2}<s_{2}'<s_{1}'$, it means that $a_{2}$, $a_{2}'$, and $a_{1}'$ (also $a_{1}$) must be three different items. However, since $a_{2}$ is already ranked at the second place, it is impossible for two different items to have higher scores than $a_{2}$. This is contradictory and cannot be true. But things are different if $s_{2}<s_{2}'=s_{1}'$. In this case, we have $a_{2}'=a_{1}'$, which means that the first row and the second row of $\mathbf{A}'$ are identical. One item is ranked at both the first position and the second position in $\mathbf{A}'$ (for example, \autoref{fig:GRO flow}). In this case, $a_{2}$ can be an arbitrary item as long as its score is no less than $a_{3}$.
\end{proof}

According to \autoref{lemma:swap} and its proof, the target model will rank an item at the same position as $\mathbf{A}'$ if the item first appears in $\mathbf{A}'$. For items with multiple appearances, those slots will be filled with other items that did not appear in $\mathbf{A}'$. Most of these filler items are, however, from the top-k list of $\mathbf{A}$, because they tend to have higher scores than other items. Therefore, \autoref{eq:swap loss} can force the model to both learn the new swap matrix $\mathbf{A}'$ and preserve the original swap matrix $\mathbf{A}$.

At last, we jointly train the target model and the student model with the following loss function:
\begin{equation}
\label{eq:final loss}
    L=L_{\text{target}}+L_{\text{student}}+\lambda L_{\text{swap}}
\end{equation}
where $L_{\text{target}}$ is the loss function of the target model, such as the cross-entropy loss; $L_{\text{student}}$ is the loss function of the student model (\autoref{eq:surrogate loss}); $L_{\text{swap}}$ is the aforementioned swap loss in \autoref{eq:swap loss}, and $\lambda$ is a hyper-parameter controlling the magnitude.

After training, the student model is discarded, and the target model can be deployed directly. The model extraction attacks aiming at stealing the target model can fail to acquire a satisfying performance, as we will show in the experiment section.

\begin{table}[t]
\caption{Dataset Statistics}
\vspace{-1em}
\begin{center}
 \begin{tabular}{|c|c|c|c|c|} 
 \hline
 Dataset & \# Users & \# Items & Avg. length & Density \\
 \hline
 ML-1M & 6040 & 3416 & 165.5 & 4.84\% \\
 \hline
 ML-20M & 138493 & 18345 & 144.3 & 0.79\% \\
 \hline
 Steam & 334542 & 13046 & 12.6 & 0.10\% \\
 \hline
\end{tabular}
\label{table:datasets}
\end{center}
\vspace{-1.3em}
\end{table}

\section{Experiments}
We apply GRO to the state-of-the-art sequential recommendation model Bert4Rec \cite{sun2019bert4rec}, a transformer-based model, on three benchmark datasets, then use the state-of-the-art model extraction attack developed by Yue et al. \cite{yue2021black} to extract the target model. We compare our GRO with three baselines and show how the target model and the surrogate model perform under each defense. First, we show the result when the surrogate model has the same architecture (Bert4Rec) as the target model. Second, we change the surrogate model into NARM \cite{li2017neural}, a GRU-based sequential recommendation model, to show the performance of GRO when the target model and the surrogate model have different architectures. At last, we show how two important hyper-parameters influence the performance, including the number of queries and the $\lambda$ in \autoref{eq:final loss}.

\begin{figure*}[t]
\centering
	\subfloat[ML-1M Target Model HR]{\includegraphics[width=0.25\linewidth]{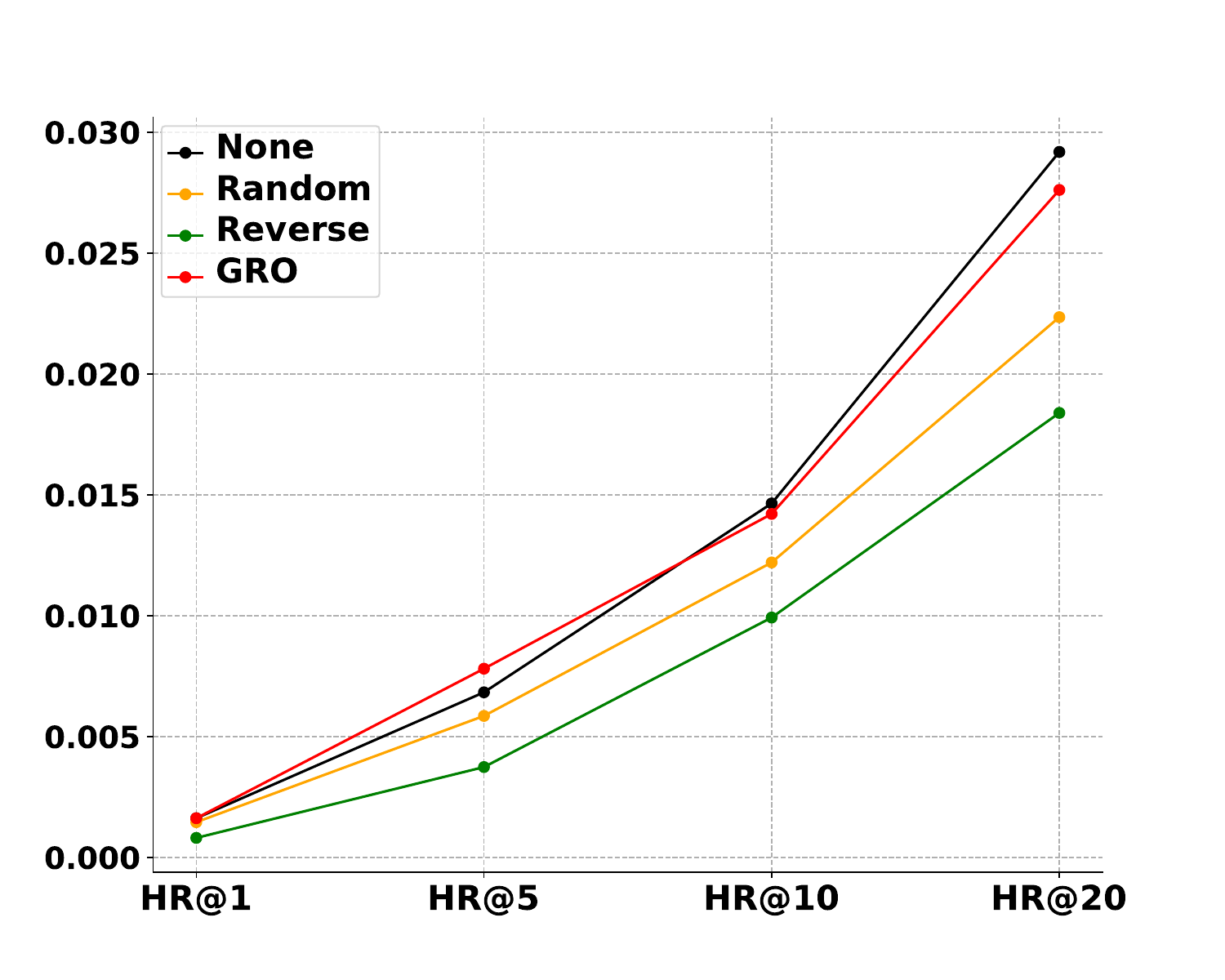}\label{fig:ml1m target hr}}
    \subfloat[ML-1M Target Model NDCG]{\includegraphics[width=0.25\linewidth]{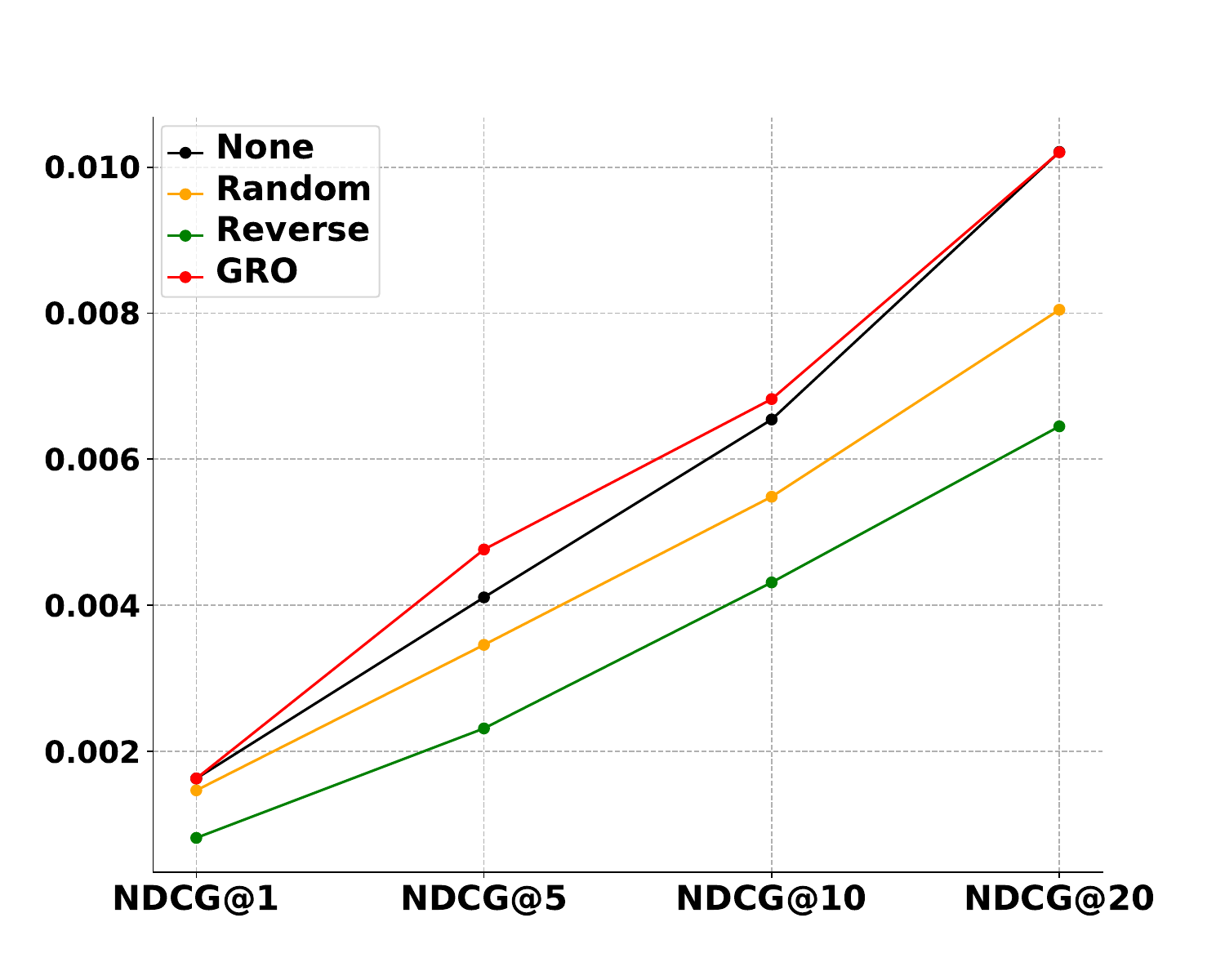}\label{fig:ml1m target ndcg}}
	\subfloat[ML-1M Surrogate Model HR]{\includegraphics[width=0.25\linewidth]{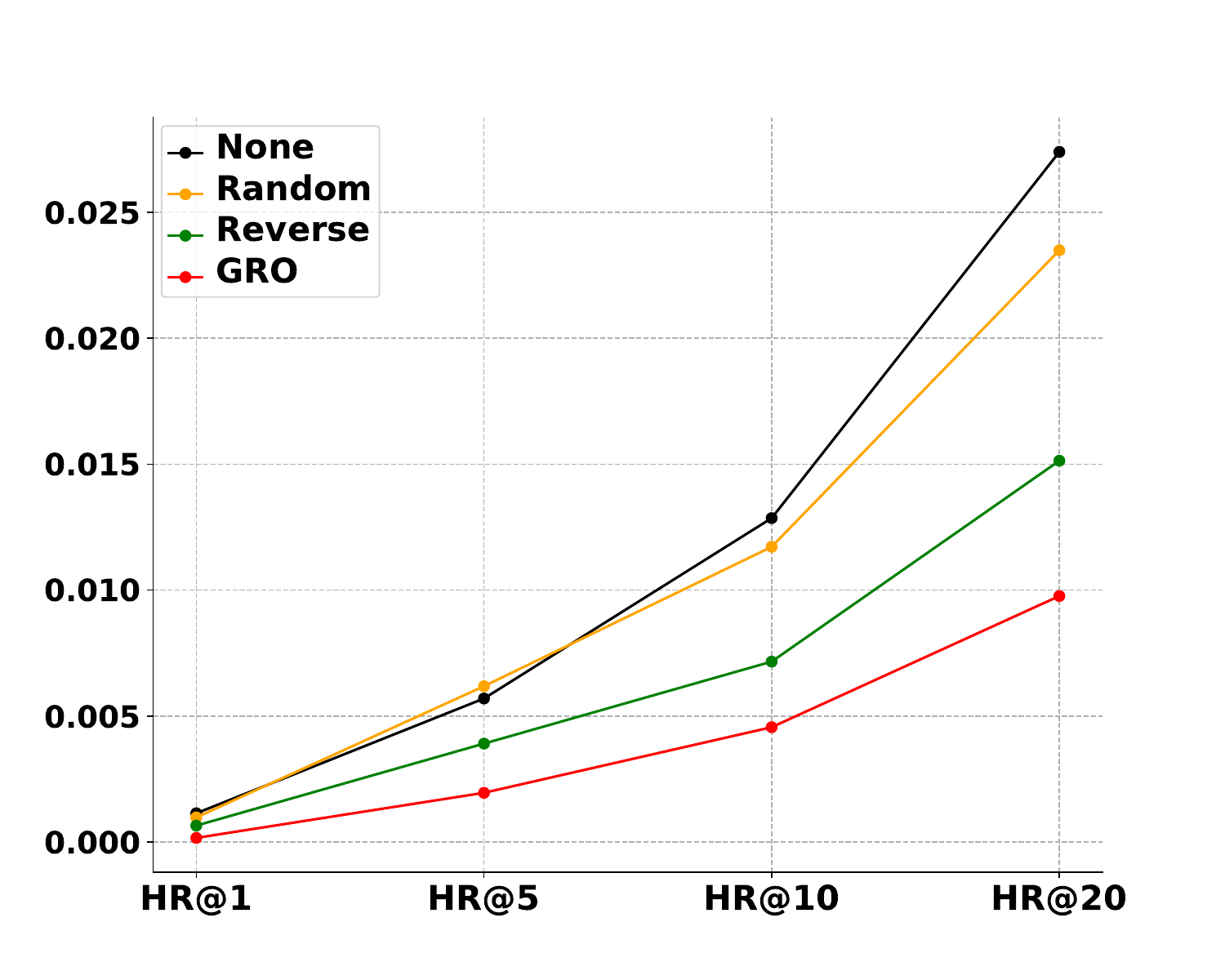}\label{fig:ml1m surrogate hr}}
	\subfloat[ML-1M Surrogate Model NDCG]{\includegraphics[width=0.25\linewidth]{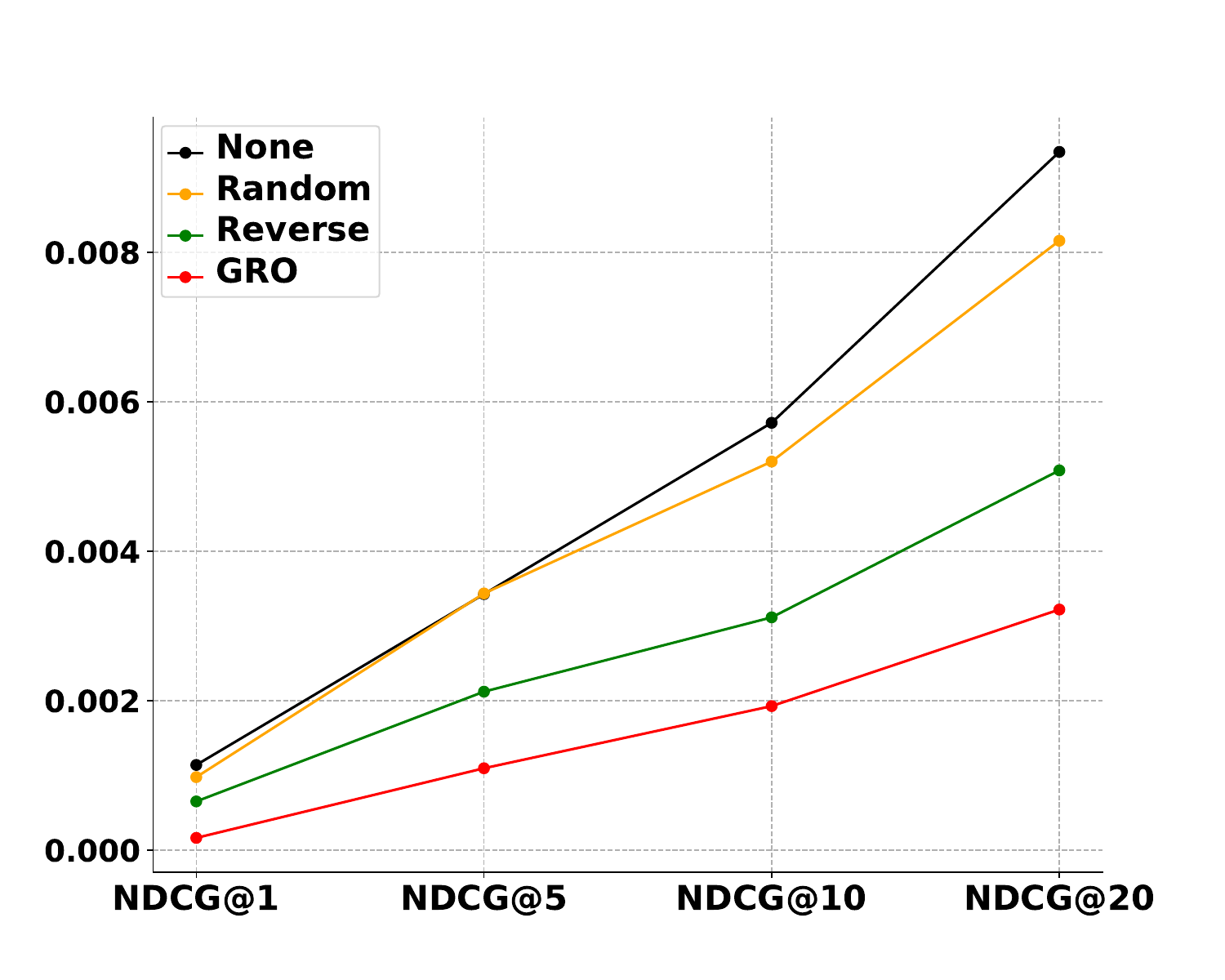}\label{fig:ml1m surrogate ndcg}}\\
	\vspace{-1.3em}
     \subfloat[ML-20M Target Model HR]{\includegraphics[width=0.25\linewidth]{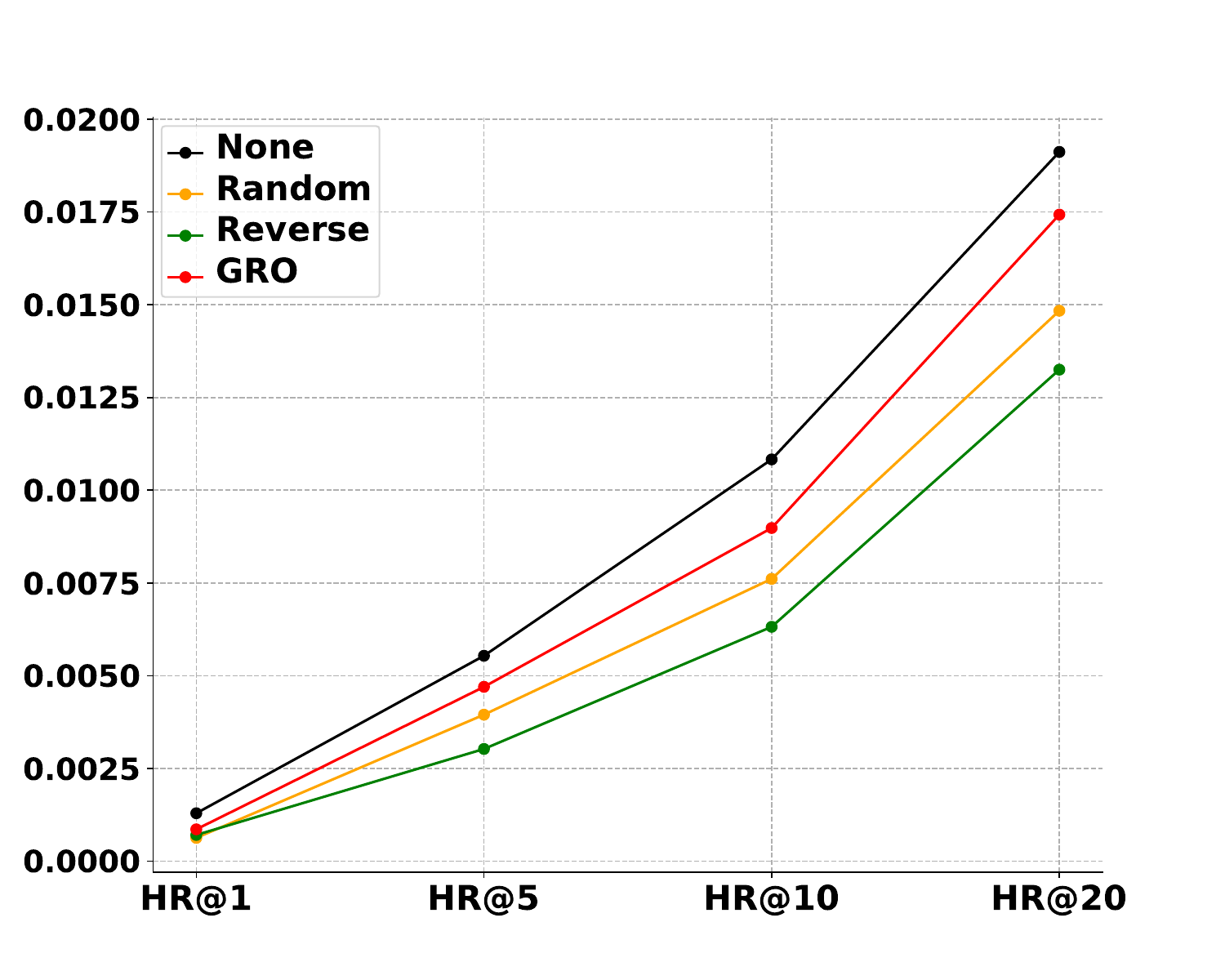}\label{fig:ml20m target hr}}
    \subfloat[ML-20M Target Model NDCG]{\includegraphics[width=0.25\linewidth]{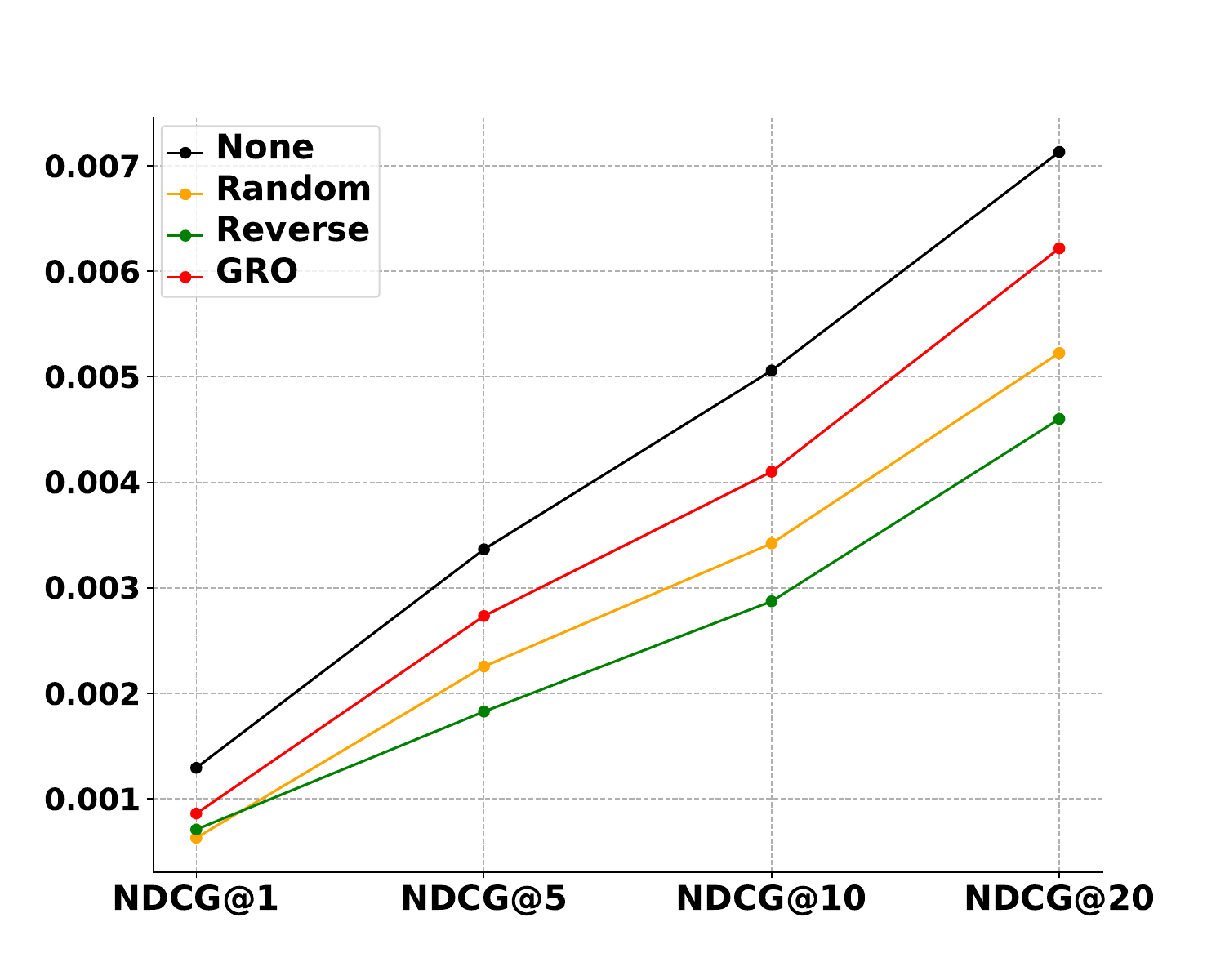}\label{fig:ml20m target ndcg}}
	\subfloat[ML-20M Surrogate Model HR]{\includegraphics[width=0.25\linewidth]{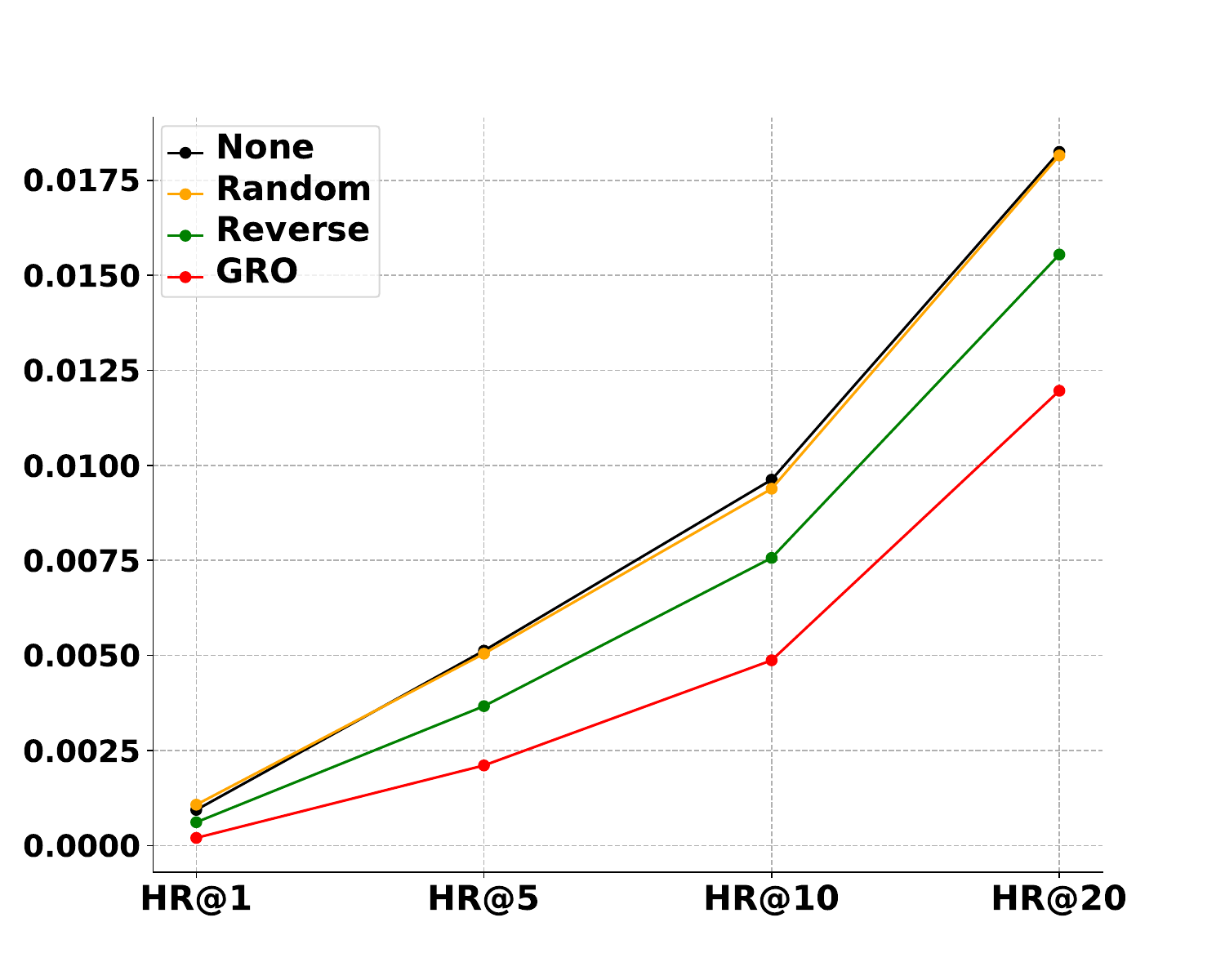}\label{fig:ml20m surrogate hr}}
	\subfloat[ML-20M Surrogate Model NDCG]{\includegraphics[width=0.25\linewidth]{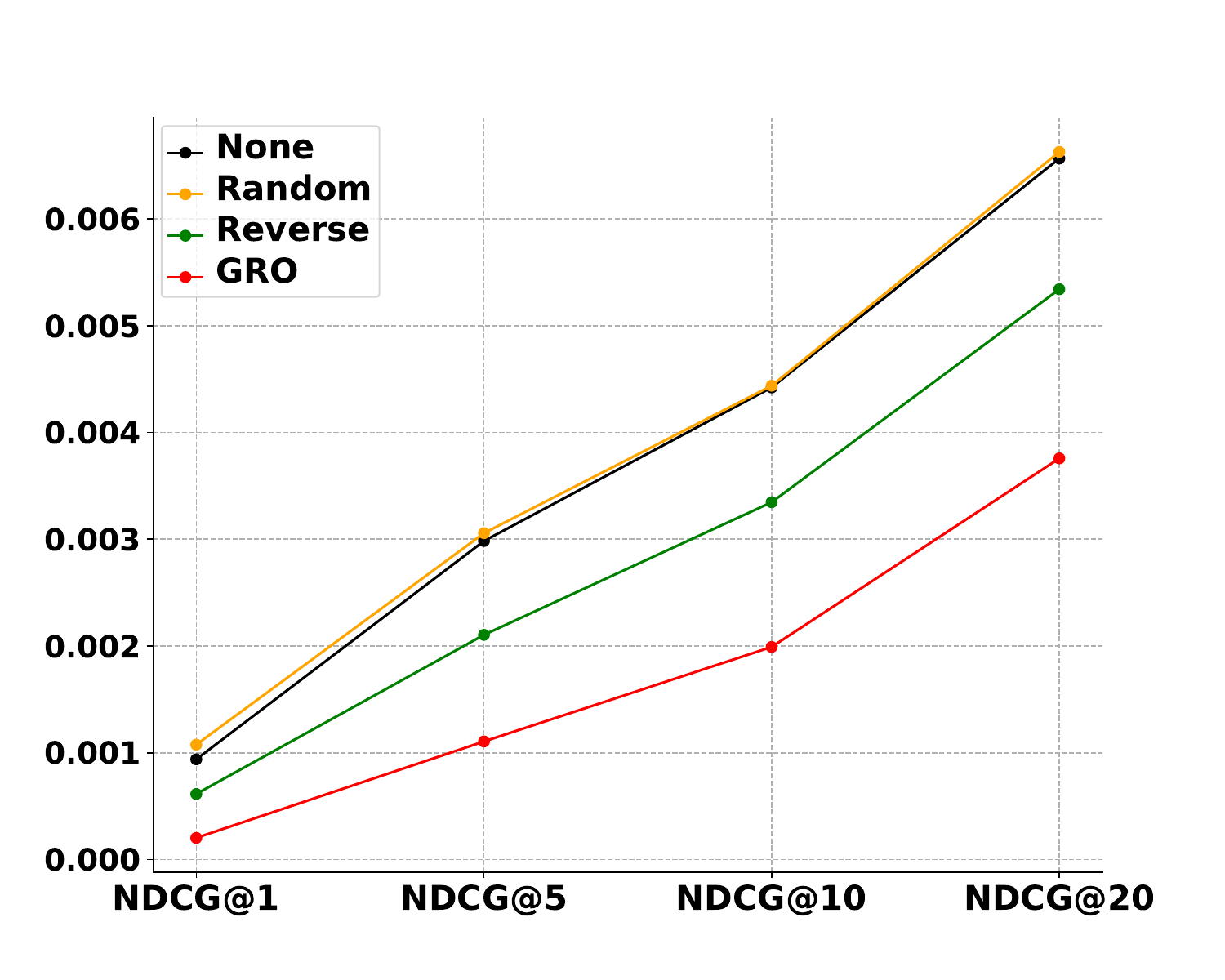}\label{fig:ml20m surrogate ndcg}} \\
    \vspace{-1.3em}
    \subfloat[Steam Target Model HR]{\includegraphics[width=0.25\linewidth]{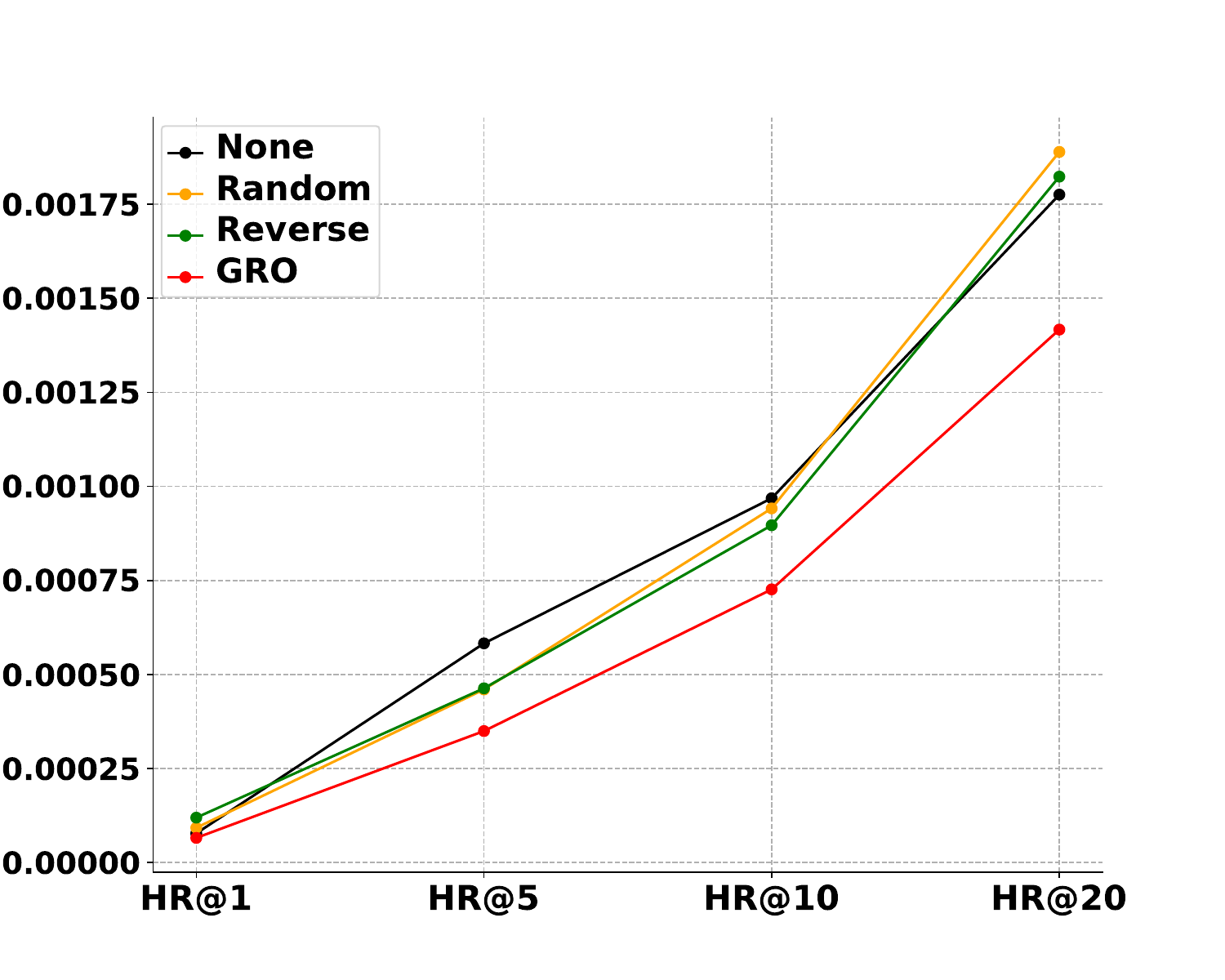}\label{fig:steam target hr}}
    \subfloat[Steam Target Model NDCG]{\includegraphics[width=0.25\linewidth]{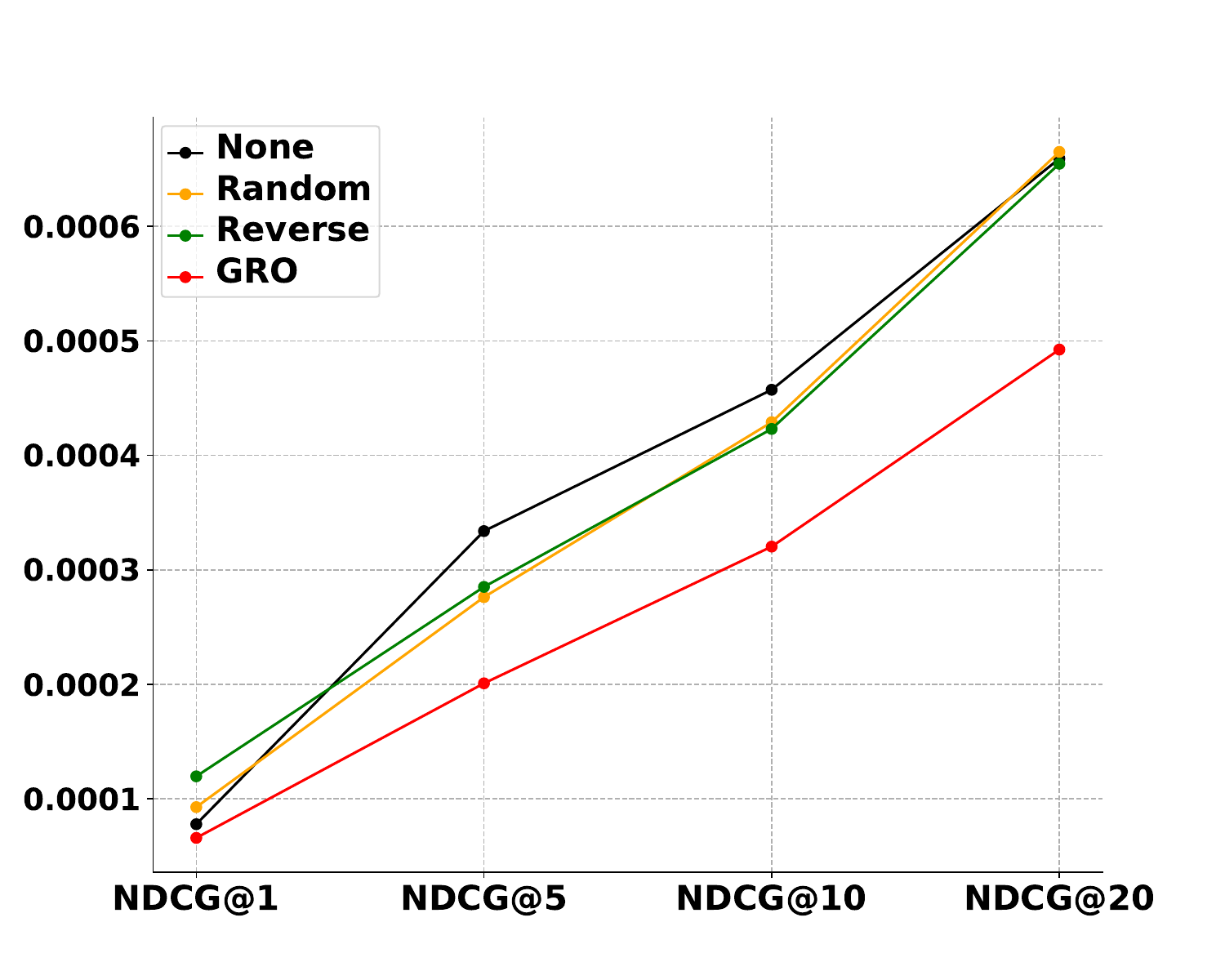}\label{fig:steam target ndcg}}
	\subfloat[Steam Surrogate Model HR]{\includegraphics[width=0.25\linewidth]{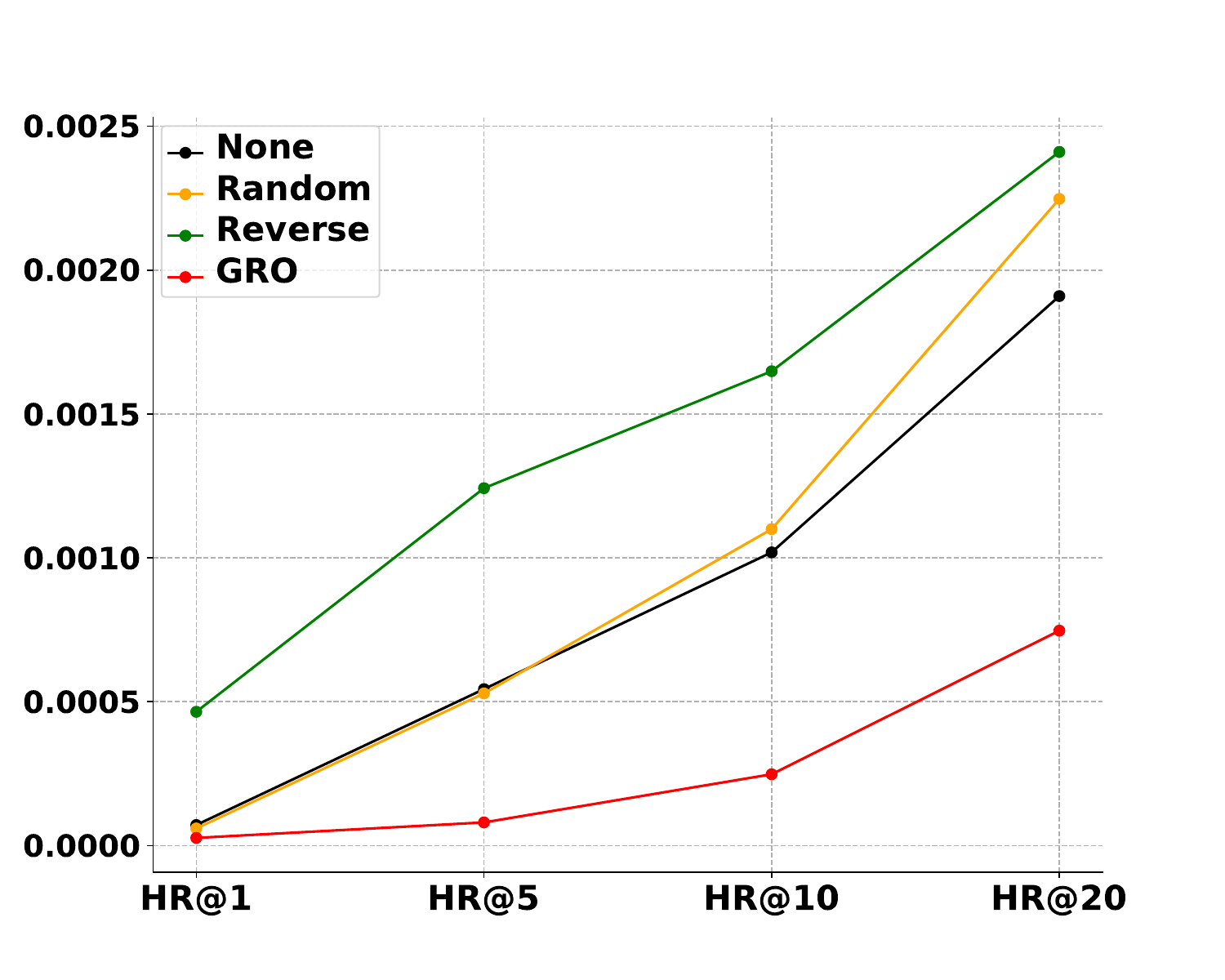}\label{fig:steam surrogate hr}}
	\subfloat[Steam Surrogate Model NDCG]{\includegraphics[width=0.25\linewidth]{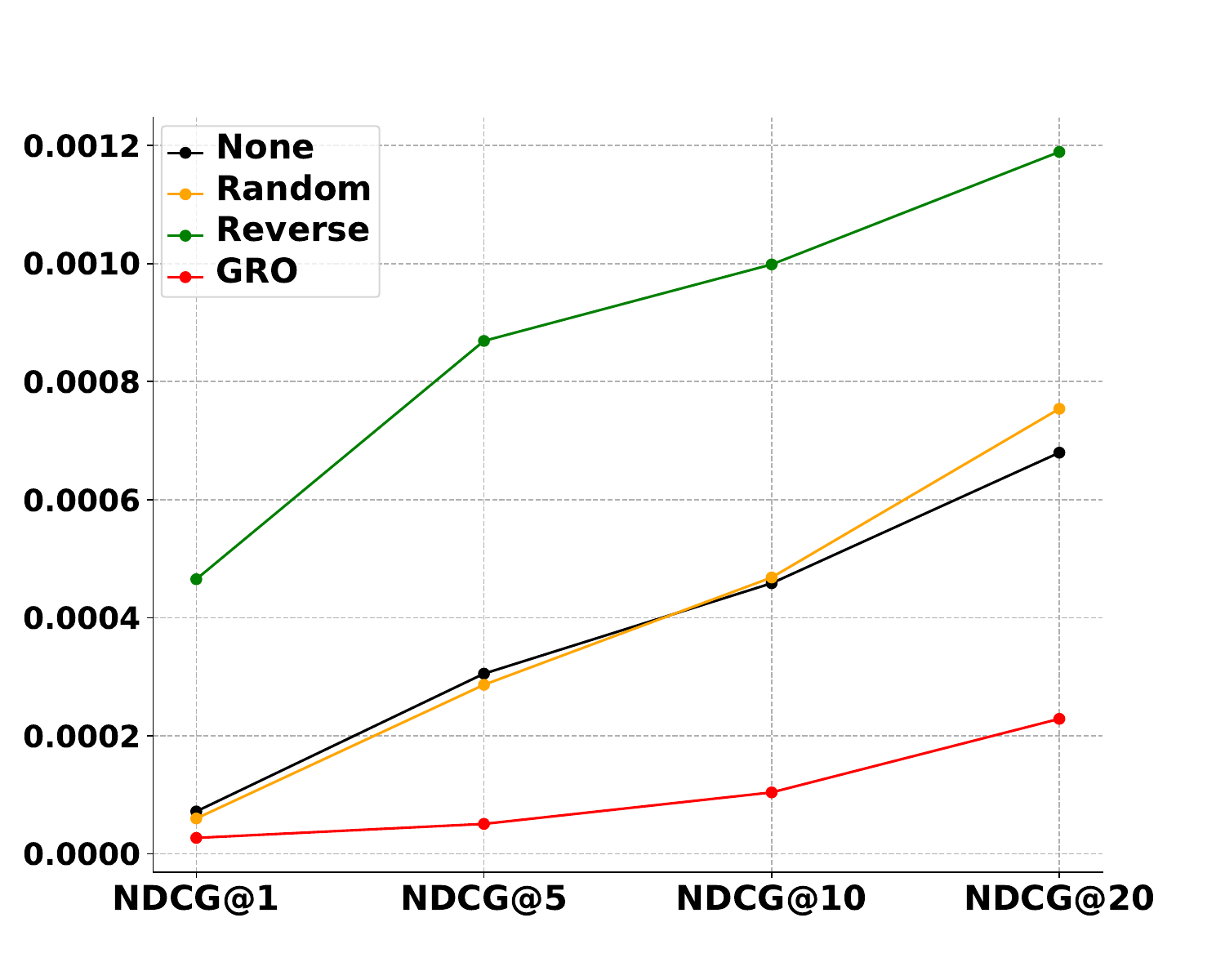}\label{fig:steam surrogate ndcg}}
\vspace{-1.3em}
\caption{Recommendation performance of the target model and the surrogate model. Both models are Bert4Rec.}
\label{fig:main exp}
\vspace{-1.3em}
\end{figure*}

\begin{figure*}[t]
\centering
     \subfloat[ML-1M Target Model HR]{\includegraphics[width=0.25\linewidth]{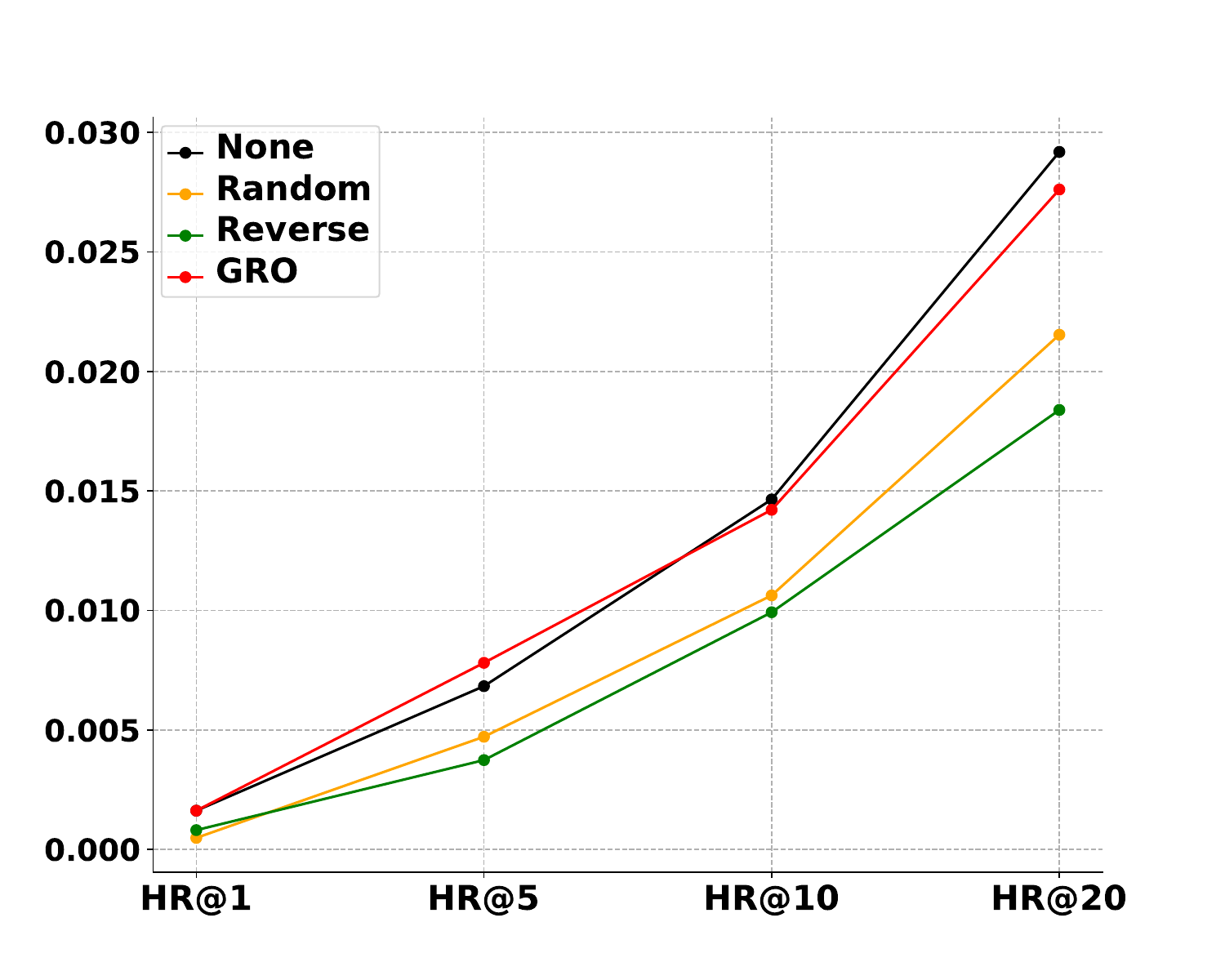}\label{fig:bert2narm ml1m target hr}}
    \subfloat[ML-1M Target Model NDCG]{\includegraphics[width=0.25\linewidth]{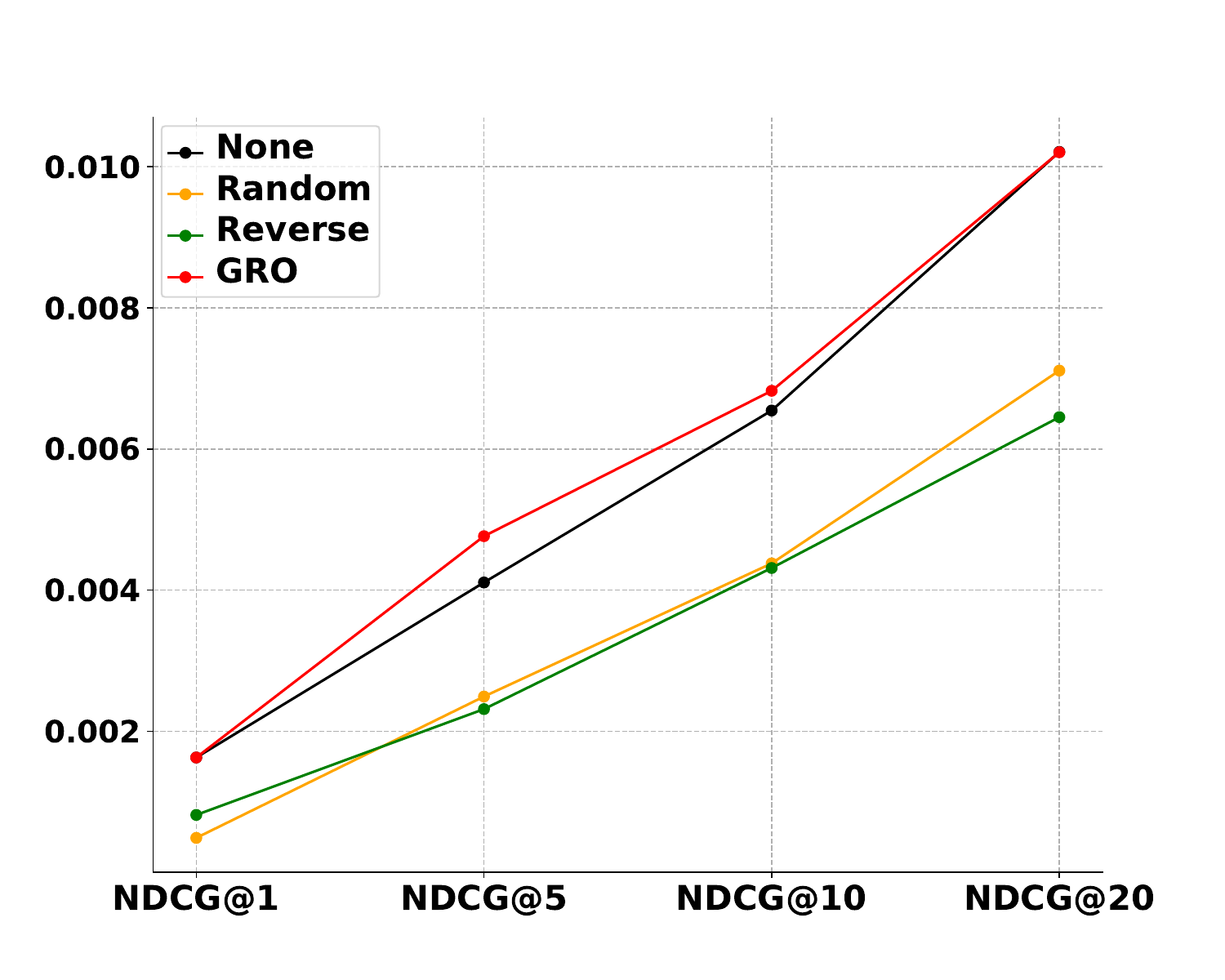}\label{fig:bert2narm ml1m target ndcg}}
	\subfloat[ML-1M Surrogate Model HR]{\includegraphics[width=0.25\linewidth]{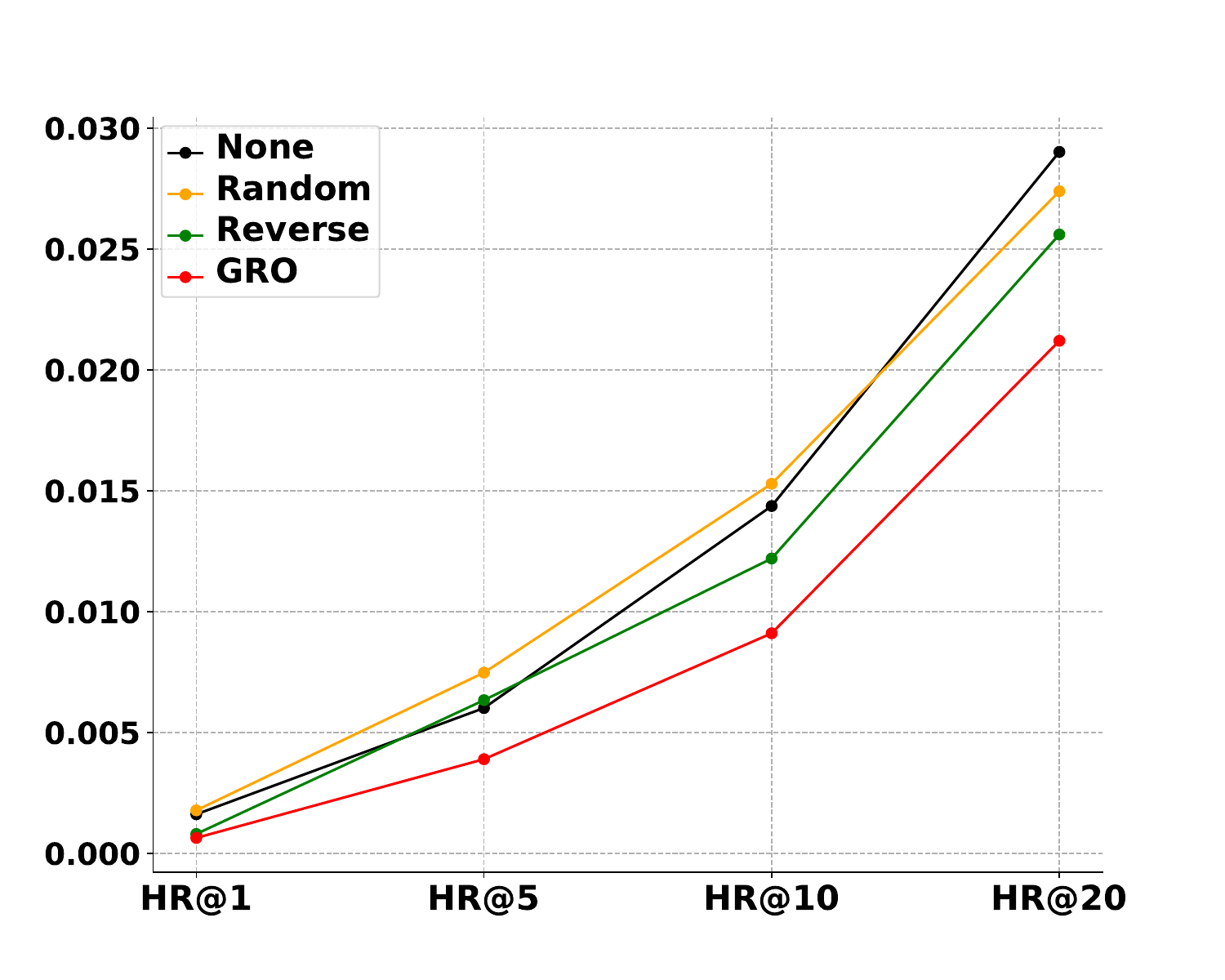}\label{fig:bert2narm ml1m surrogate hr}}
	\subfloat[ML-1M Surrogate Model NDCG]{\includegraphics[width=0.25\linewidth]{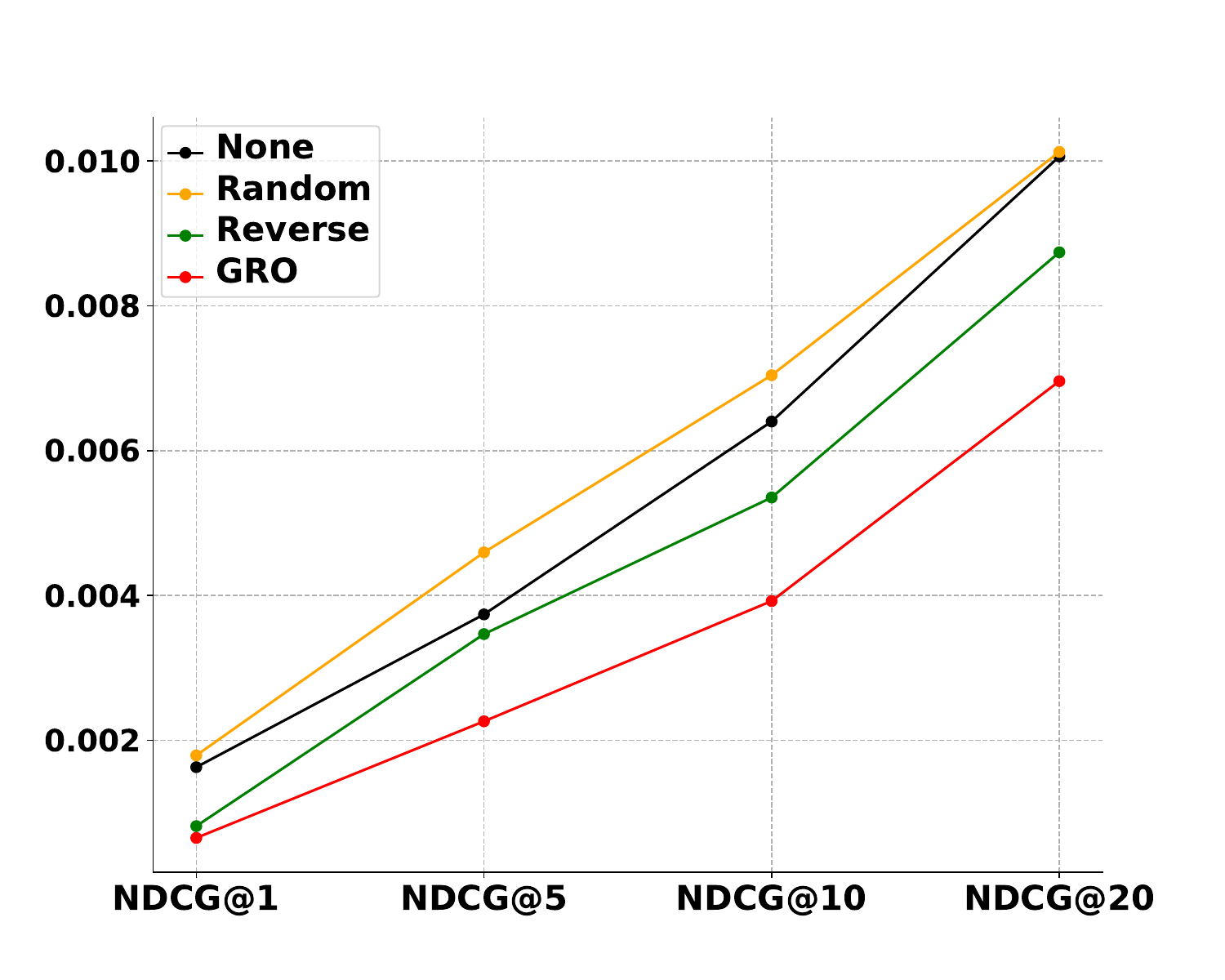}\label{fig:bert2narm ml1m surrogate ndcg}} \\
    \vspace{-1.3em}
    \subfloat[Steam Target Model HR]{\includegraphics[width=0.25\linewidth]{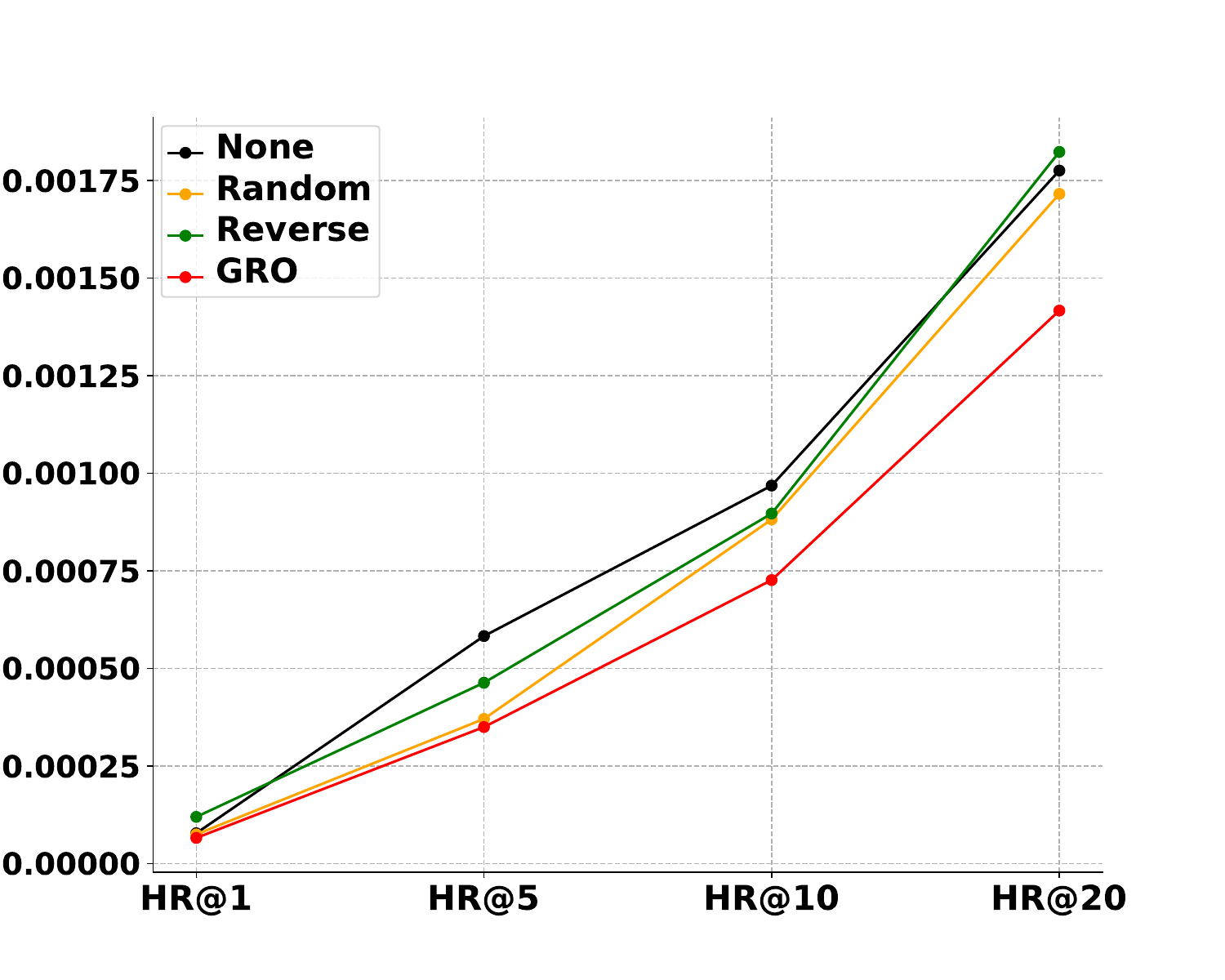}\label{fig:bert2narm steam target hr}}
    \subfloat[Steam Target Model NDCG]{\includegraphics[width=0.25\linewidth]{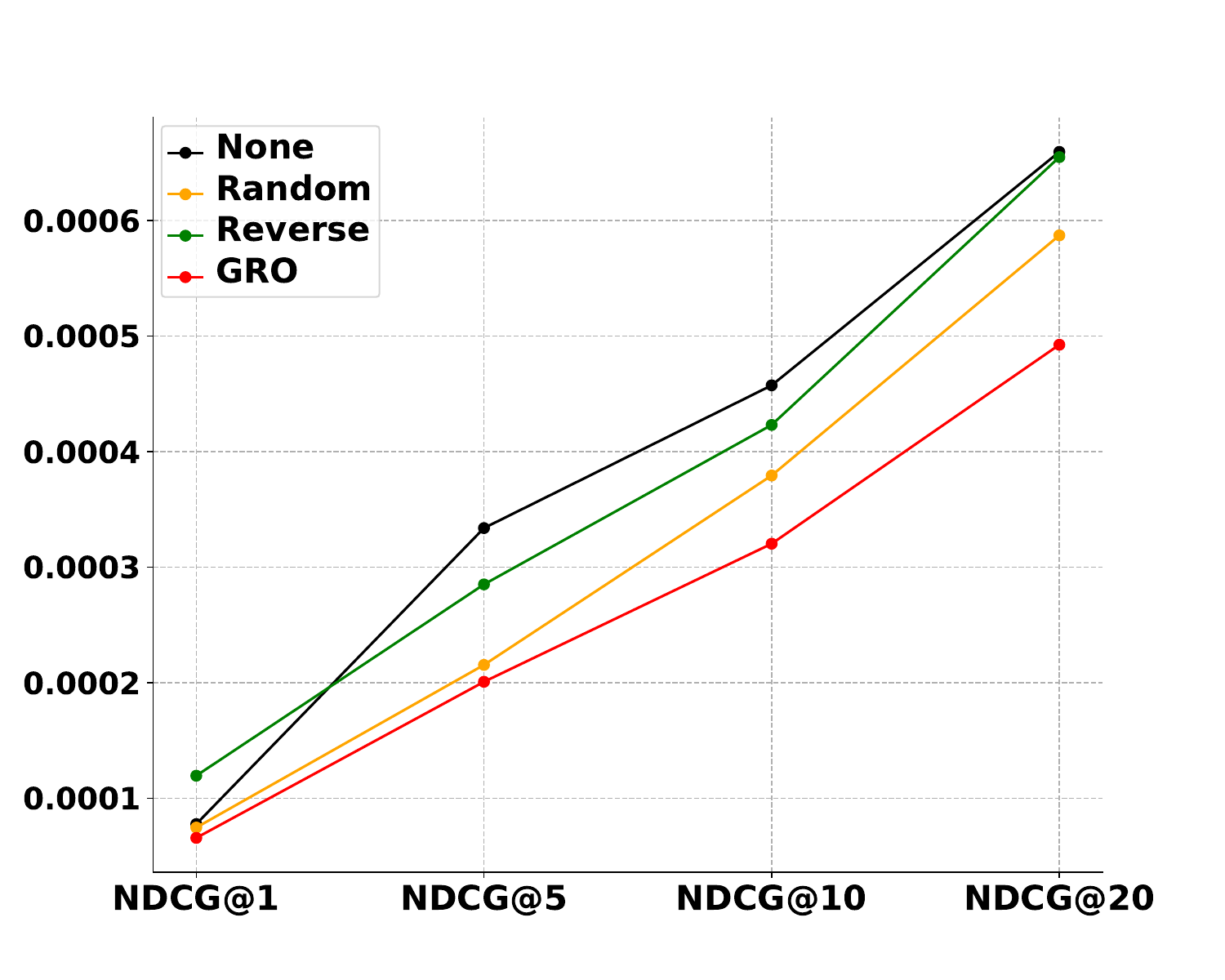}\label{fig:bert2narm steam target ndcg}}
	\subfloat[Steam Surrogate Model HR]{\includegraphics[width=0.25\linewidth]{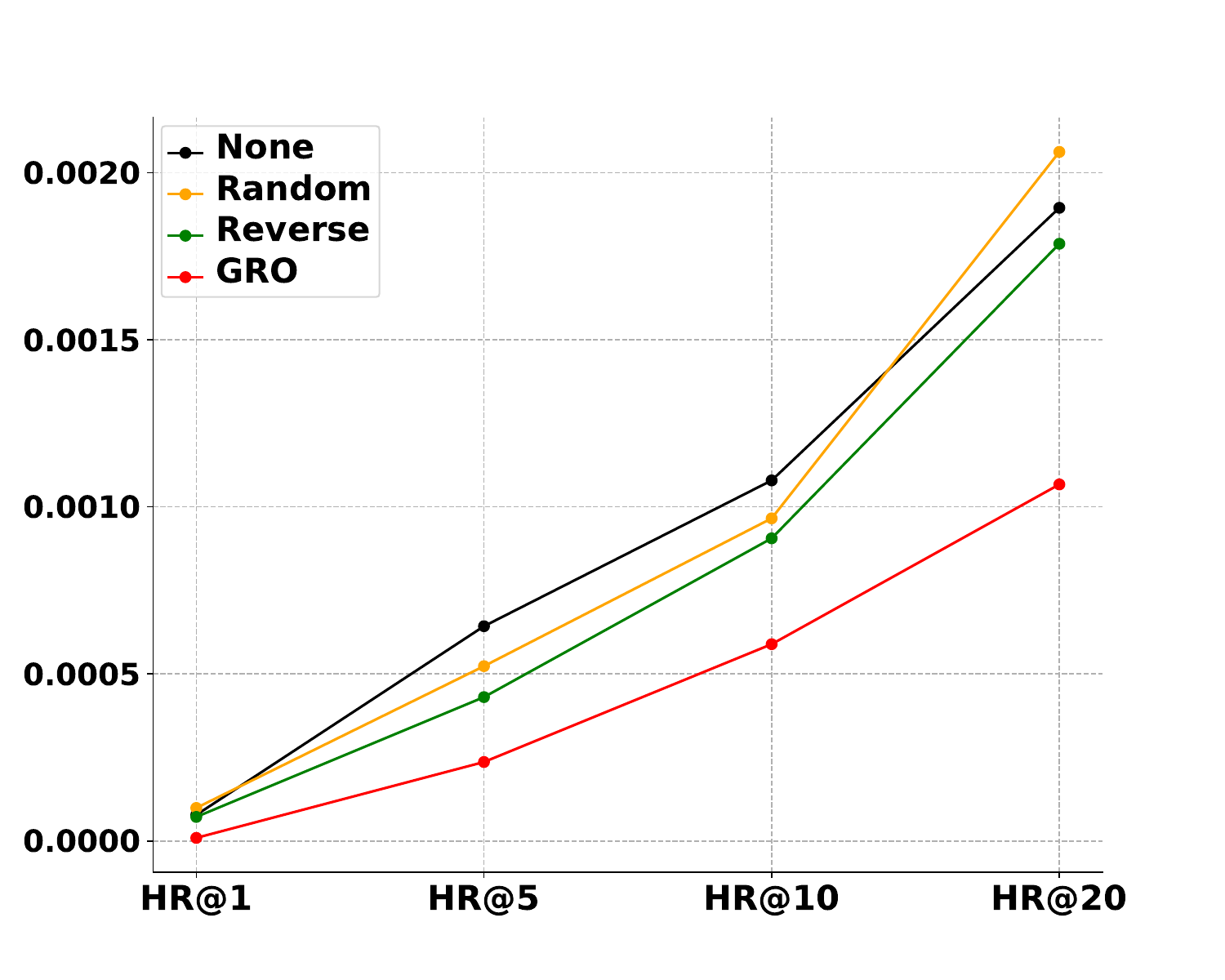}\label{fig:bert2narm steam surrogate hr}}
	\subfloat[Steam Surrogate Model NDCG]{\includegraphics[width=0.25\linewidth]{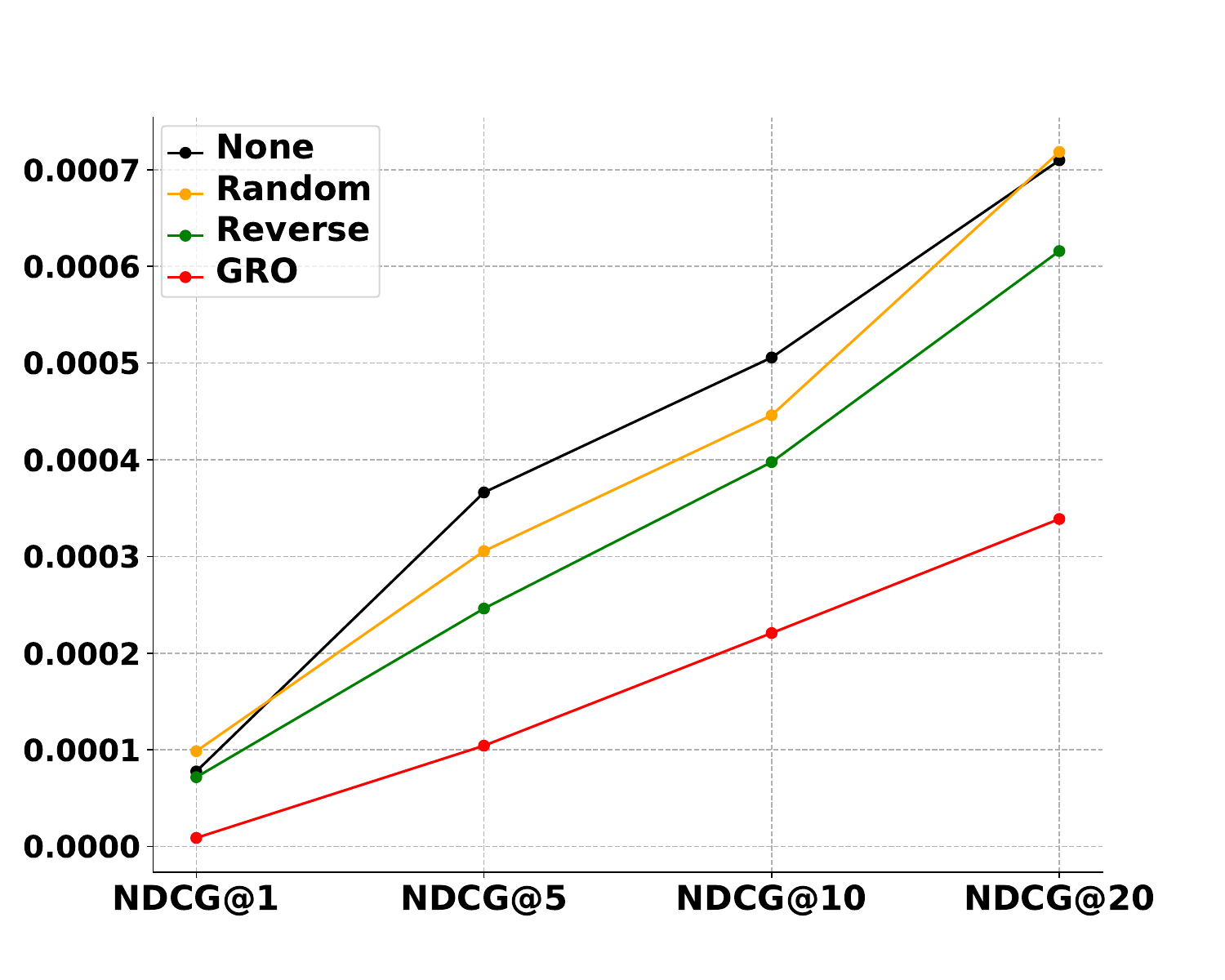}\label{fig:bert2narm steam surrogate ndcg}}
\vspace{-1.3em}
\caption{Recommendation performance of the target model and the surrogate model. The target model is Bert4Rec, while the surrogate model is NARM.}
\label{fig:different models}
\vspace{-1.3em}
\end{figure*}

\subsection{Datasets}
We use three benchmark datasets, namely MovieLens-1M, MovieLens-20M\footnote{\url{https://grouplens.org/datasets/movielens/}}, and Steam\footnote{\url{https://cseweb.ucsd.edu/~jmcauley/datasets.html\#steam_data}}. The statistics of the datasets are summarized in \autoref{table:datasets}. Each user has one sequence of interaction history. We follow previous works \cite{sun2019bert4rec, yue2021black} to use leave-one-out evaluation. The last two items in each sequence are used as the validation set and the test set respectively.


\subsection{Baselines}
We compare GRO with three baselines, namely None, Random and Reverse.
\begin{itemize}
    \item \textbf{None}. We do not apply any defense method to the target model.
    \item \textbf{Random}. We randomly shuffle the top-k recommendation list returned by the target model. The shuffled list is the final output of the recommender system and is what the attacker observes.
    \item \textbf{Reverse}. We reverse the ordering of the top-k recommendation list. For example, the k-th item becomes the 1st, and the (k-1)-th item becomes the 2nd. This is a strong baseline to defend against model extraction attacks while preserving the utility of the target model. It can be regarded as one of the worst cases of Random.
\end{itemize}

\begin{figure*}[t]
\centering
     \subfloat[ML-20M Target Model HR]{\includegraphics[width=0.25\linewidth]{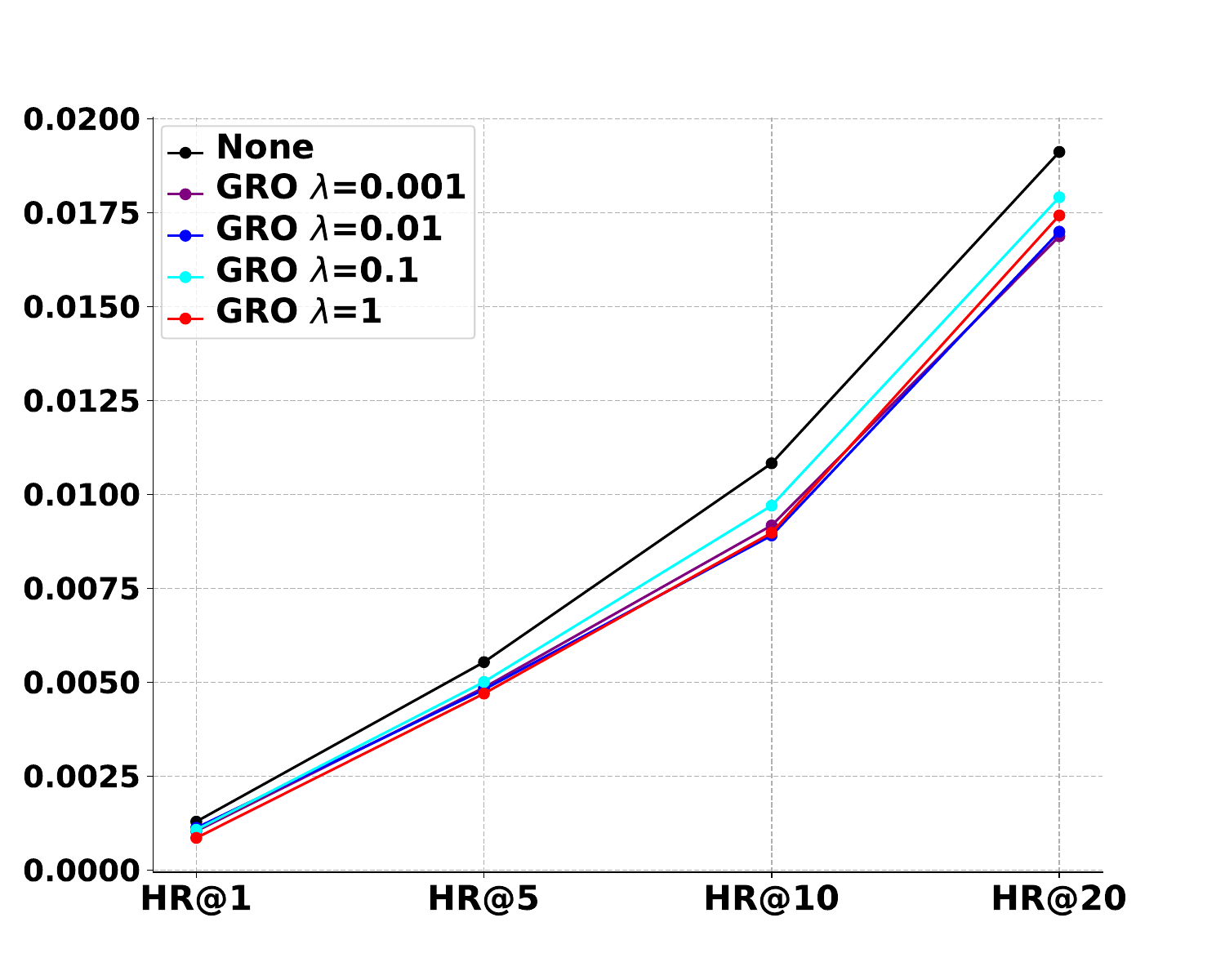}\label{fig:ml20m target hr lambda}}
    \subfloat[ML-20M Target Model NDCG]{\includegraphics[width=0.25\linewidth]{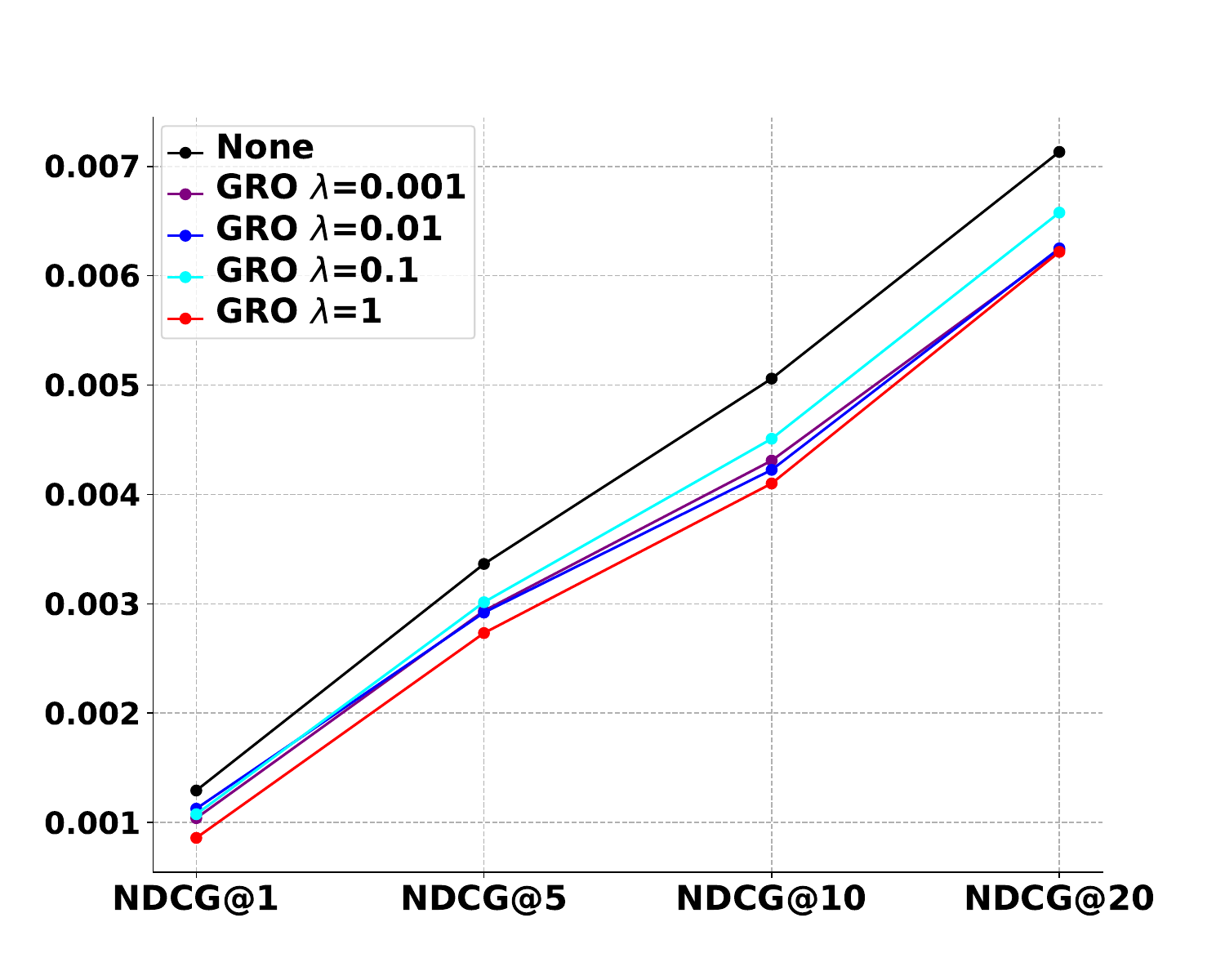}\label{fig:ml20m target ndcg lambda}}
	\subfloat[ML-20M Surrogate Model HR]{\includegraphics[width=0.25\linewidth]{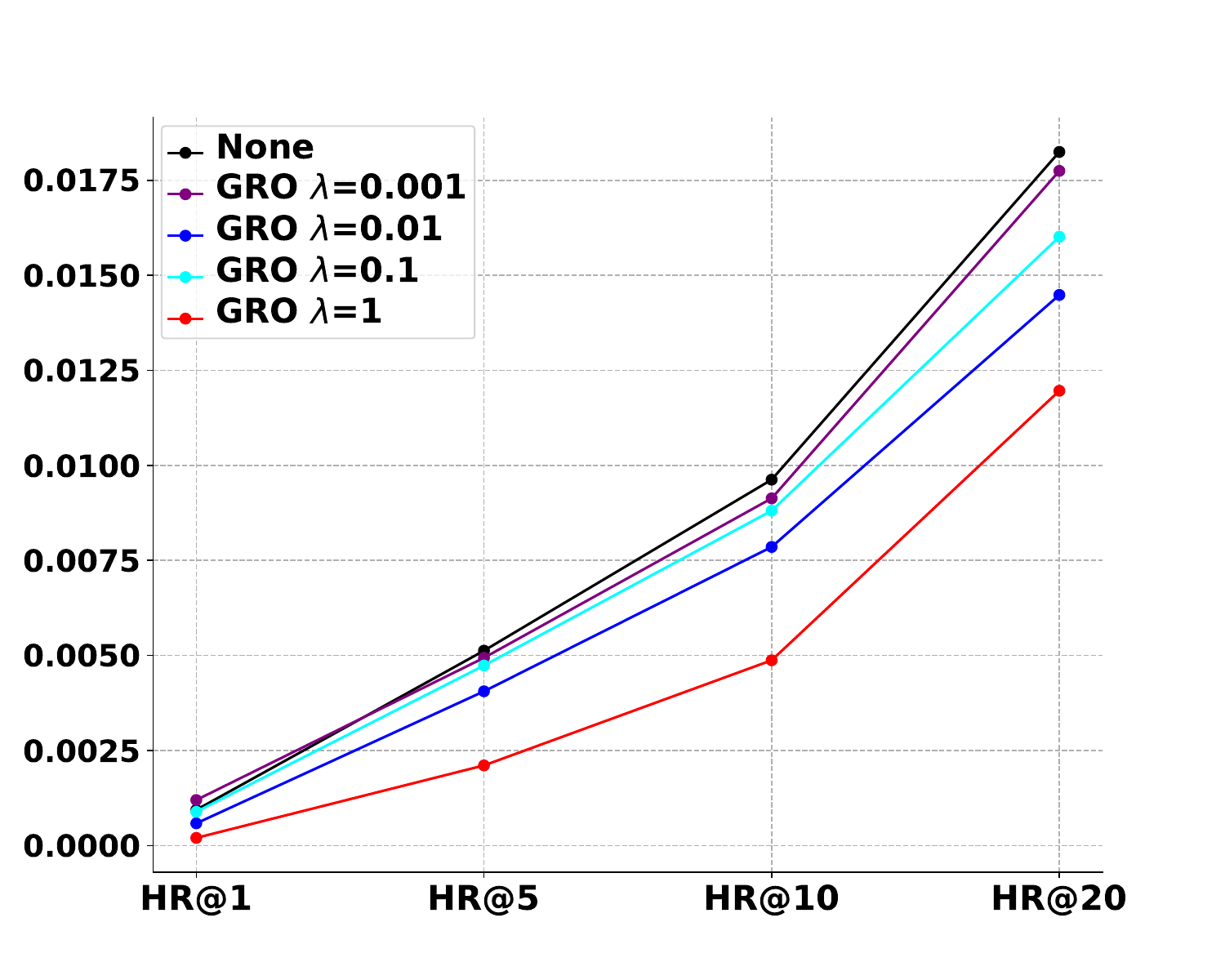}\label{fig:ml20m surrogate hr lambda}}
	\subfloat[ML-20M Surrogate Model NDCG]{\includegraphics[width=0.25\linewidth]{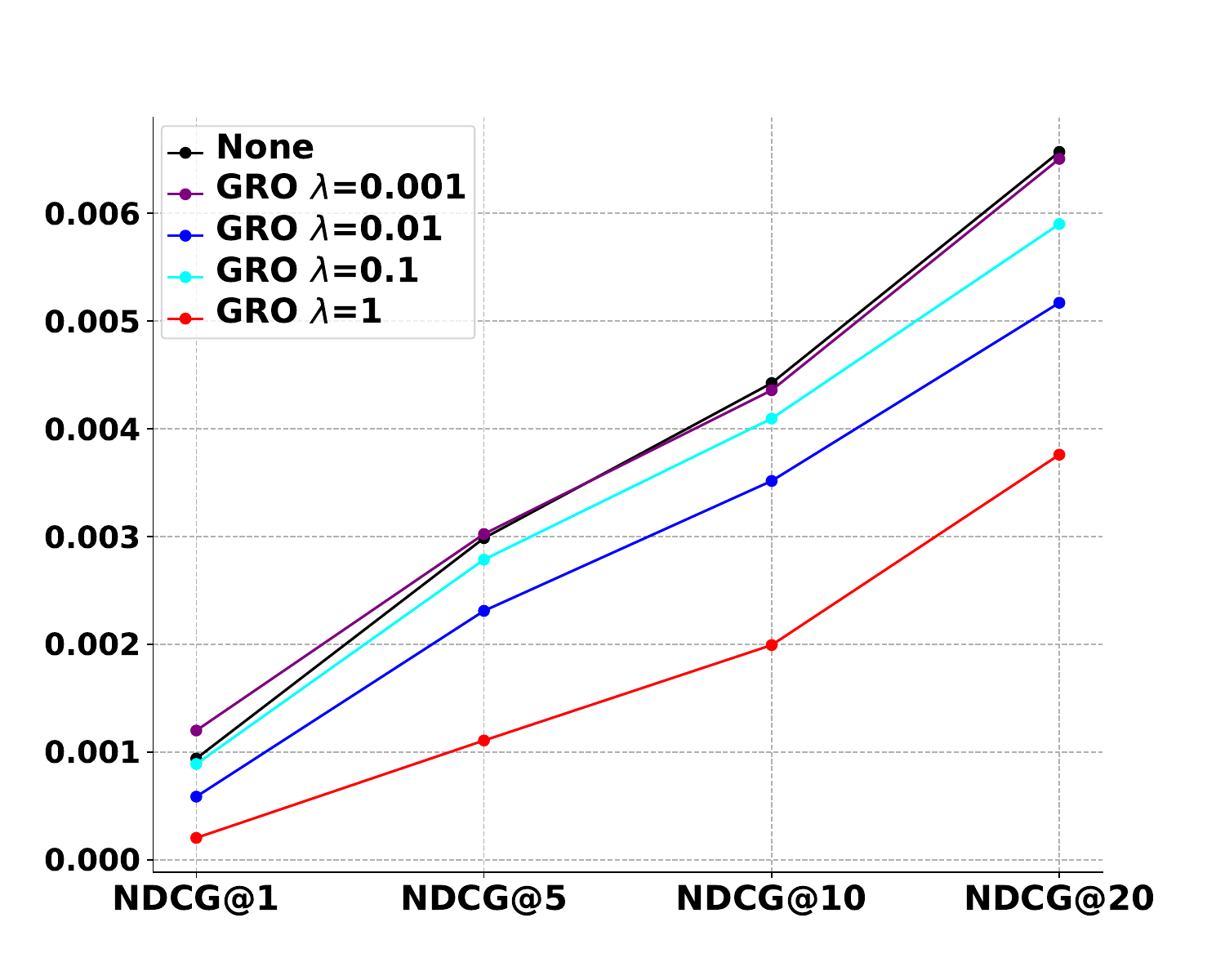}\label{fig:ml20m surrogate ndcg lambda}} \\
    \vspace{-1.3em}
    \subfloat[Steam Target Model HR]{\includegraphics[width=0.25\linewidth]{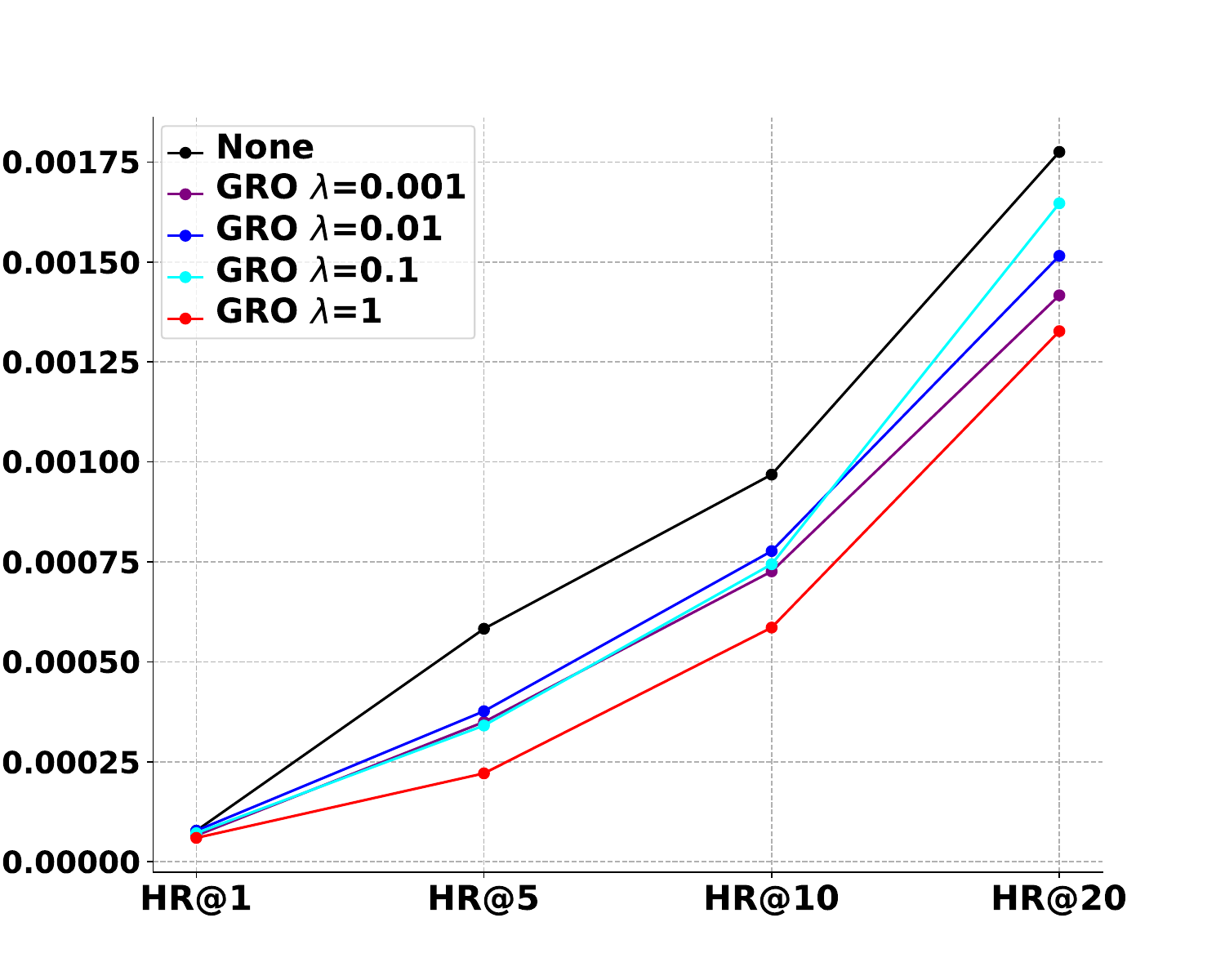}\label{fig:steam target hr lambda}}
    \subfloat[Steam Target Model NDCG]{\includegraphics[width=0.25\linewidth]{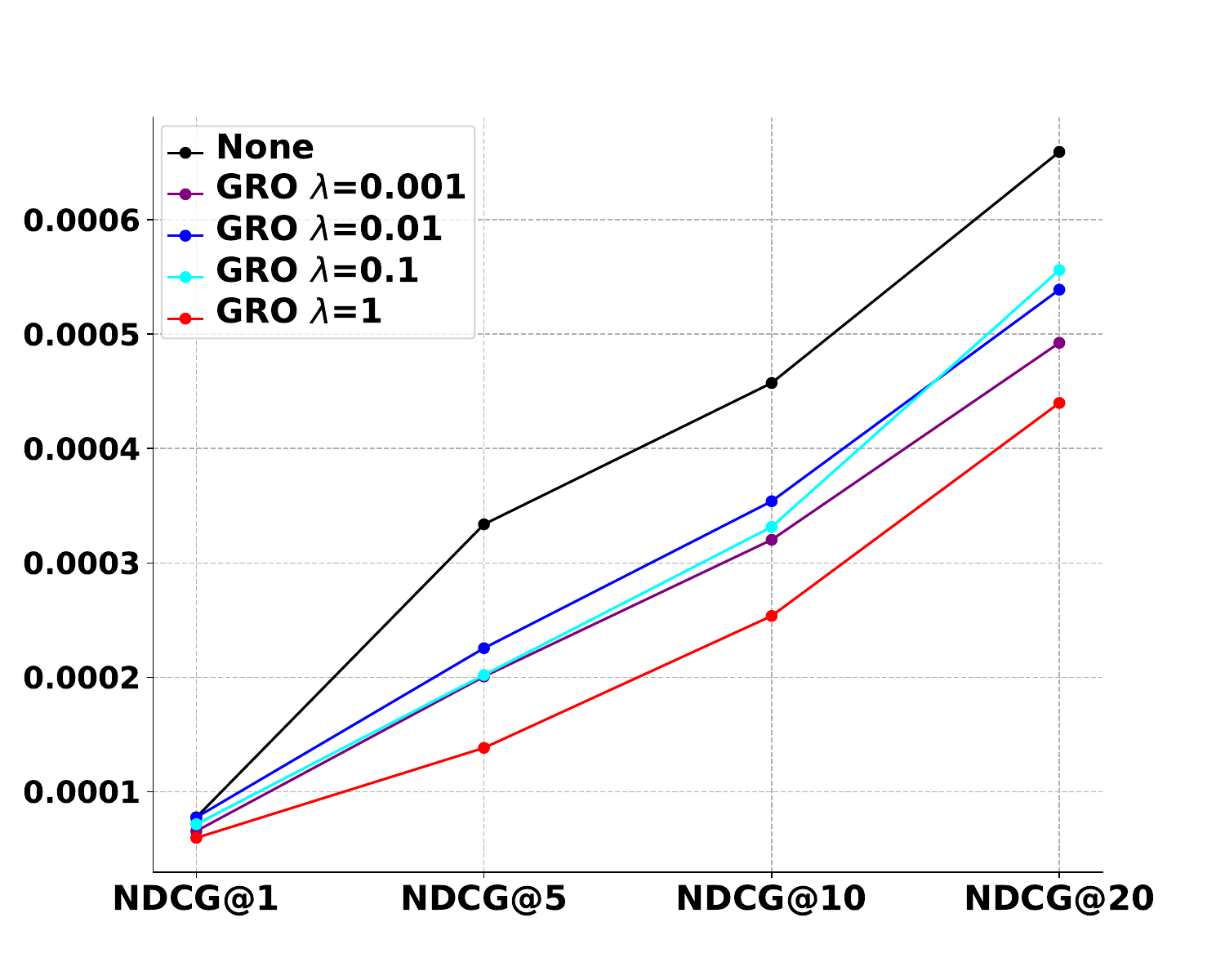}\label{fig:steam target ndcg lambda}}
	\subfloat[Steam Surrogate Model HR]{\includegraphics[width=0.25\linewidth]{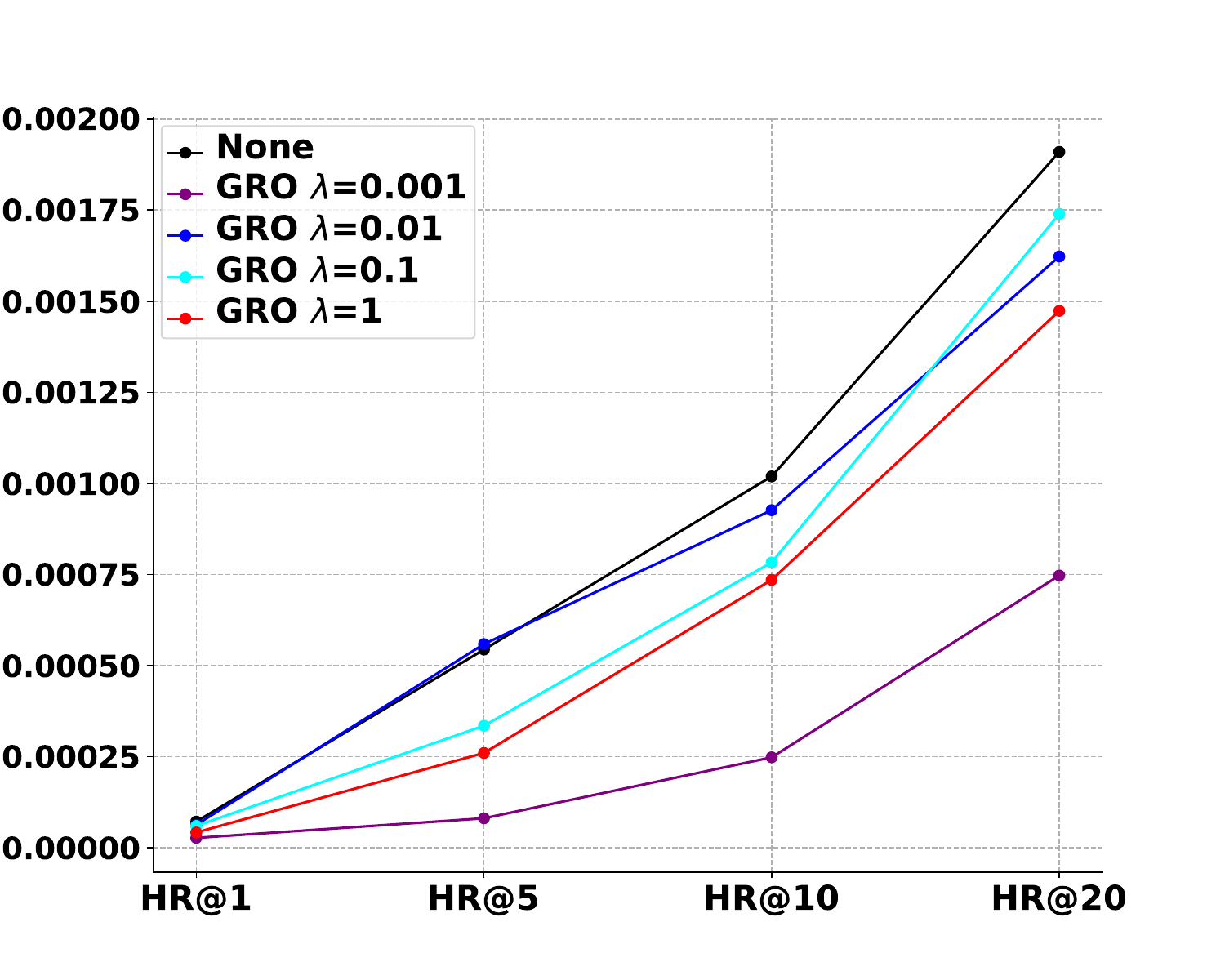}\label{fig:steam surrogate hr lambda}}
	\subfloat[Steam Surrogate Model NDCG]{\includegraphics[width=0.25\linewidth]{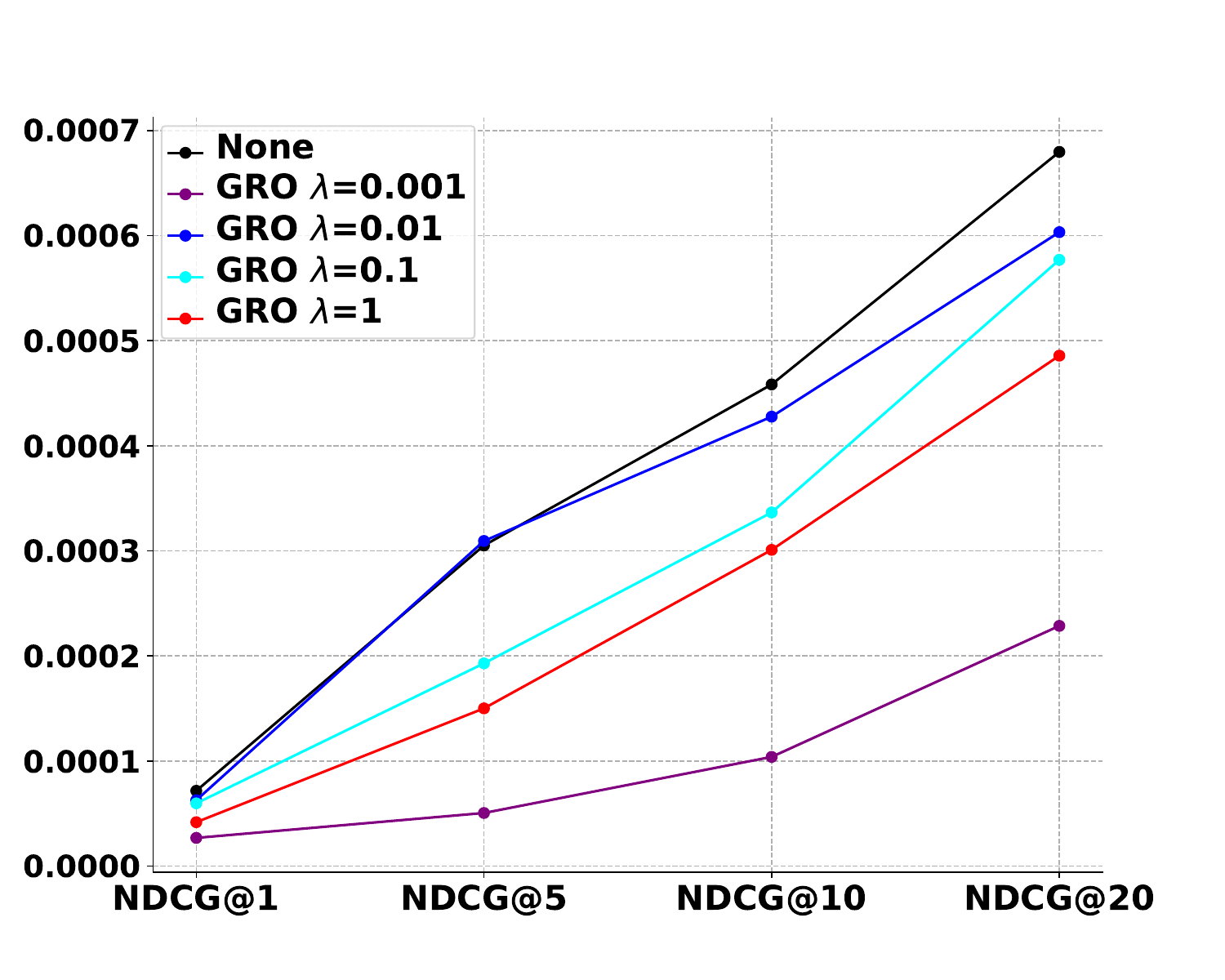}\label{fig:steam surrogate ndcg lambda}}
\vspace{-1.3em}
\caption{Recommendation performance of the target model and the surrogate model with different $\lambda$ values.}
\label{fig:lambda}
\vspace{-1.3em}
\end{figure*}

\begin{figure*}[t]
\centering
     \subfloat[Target Model HR, 1000 seqs]{\includegraphics[width=0.25\linewidth]{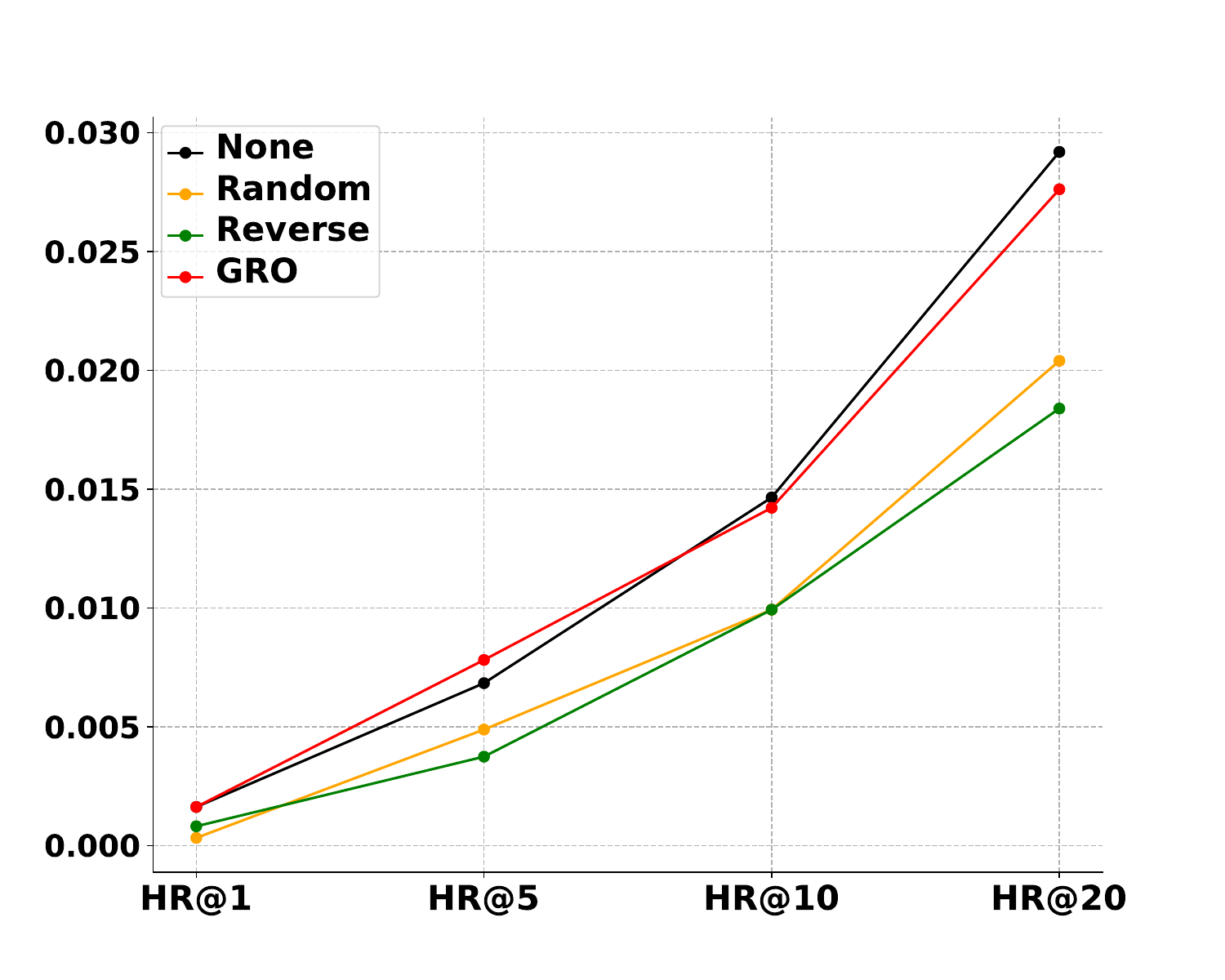}\label{fig:1000 target hr}}
	\subfloat[Surrogate Model HR, 1000 seqs]{\includegraphics[width=0.25\linewidth]{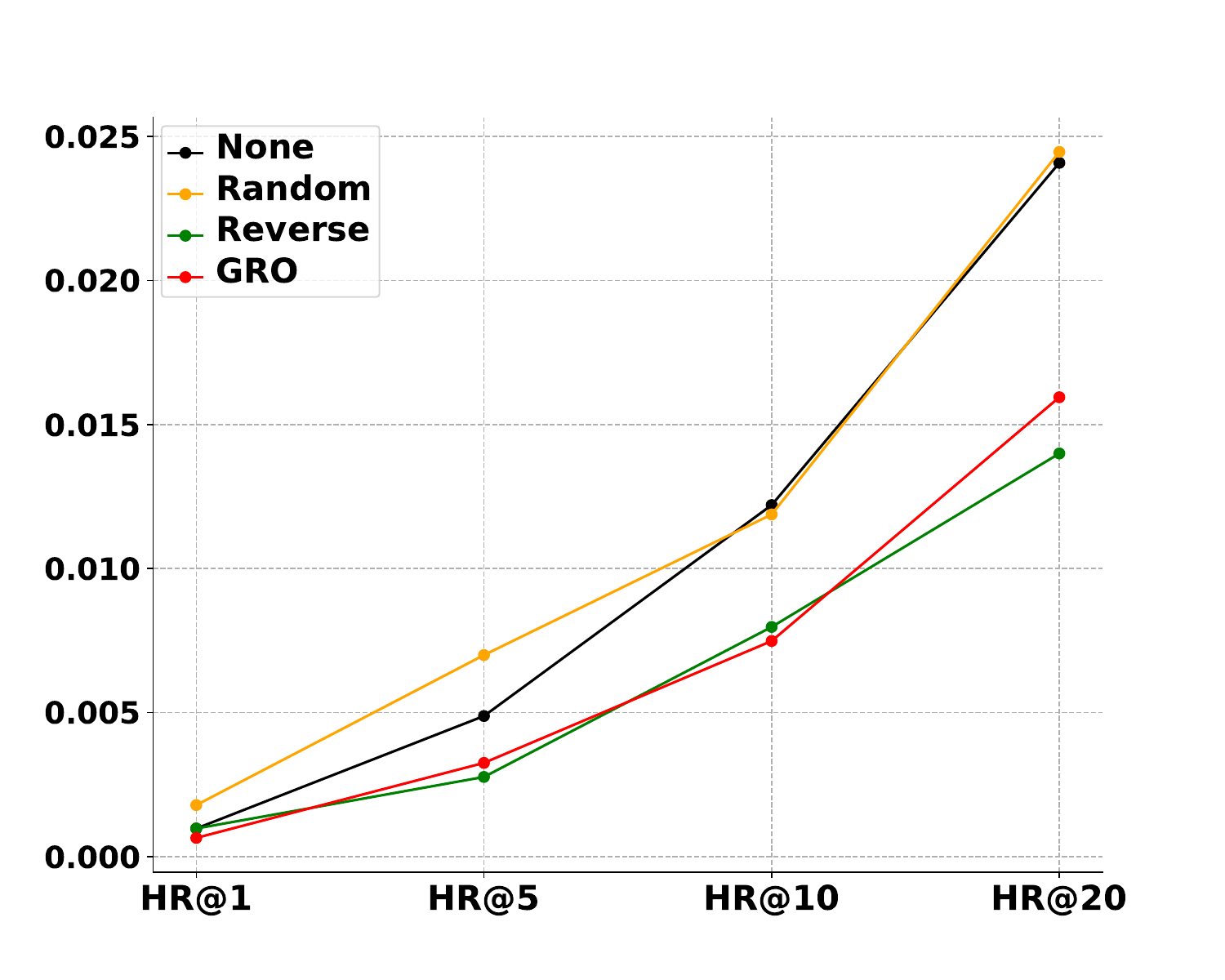}\label{fig:1000 surrogate hr}}
    \subfloat[Target Model HR, 5000 seqs]{\includegraphics[width=0.25\linewidth]{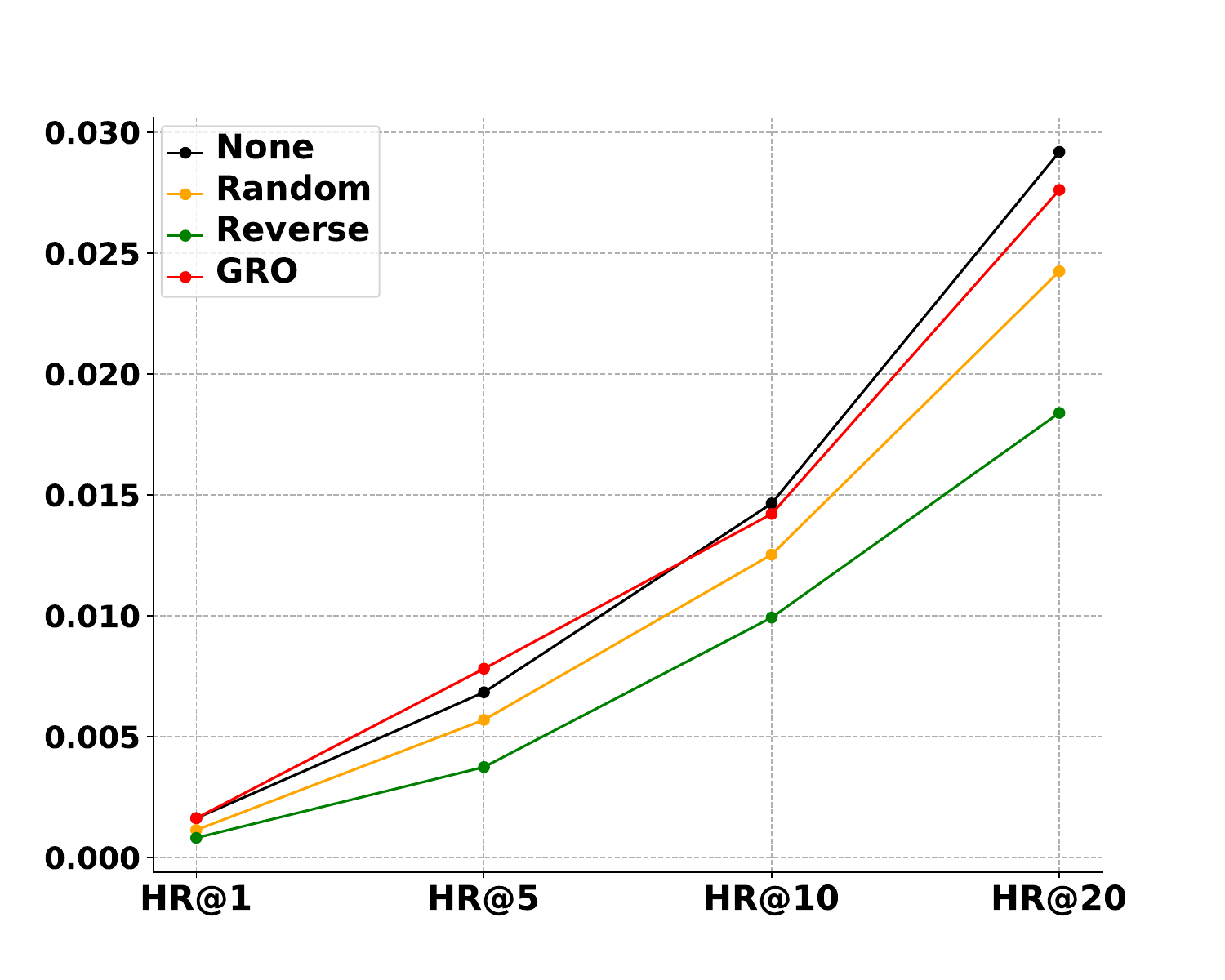}\label{fig:5000 target hr}}
	\subfloat[Surrogate Model HR, 5000 seqs]{\includegraphics[width=0.25\linewidth]{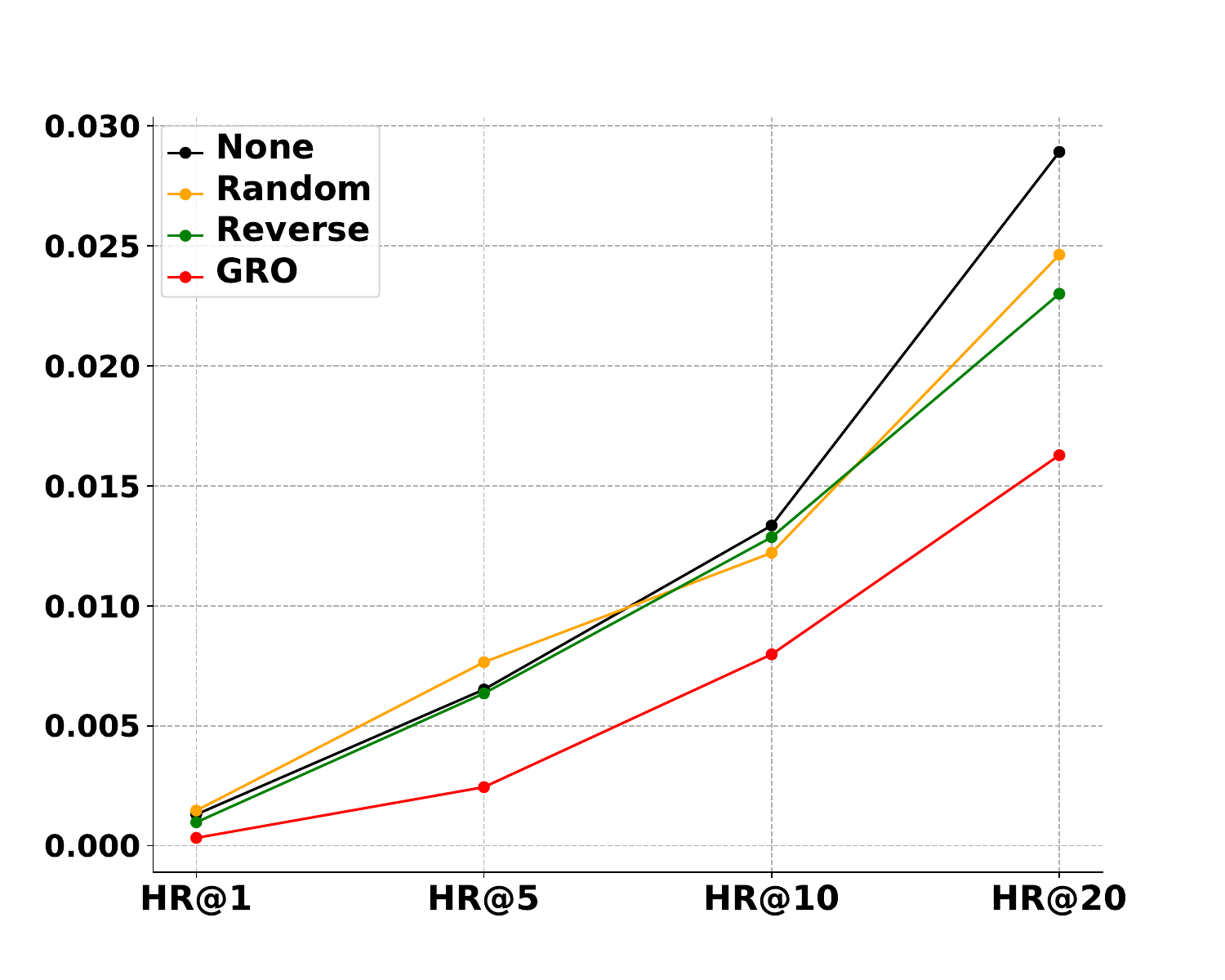}\label{fig:5000 surrogate hr}}
\vspace{-1.3em}
\caption{Recommendation performance of the target model and the surrogate model on ML-1M with different number of sequences.}
\label{fig:number seqs}
\vspace{-1.3em}
\end{figure*}

\subsection{Evaluation Metrics}
We use the Hit Ratio@k (HR@k) and Normalized Discounted Cumulative Gain@k (NDCG@k) as the evaluation metrics to evaluate the utility of the recommender systems. The k here denotes the top-k recommendation. For example, HR@10 means the hit ratio of the top-10 recommendation list. HR measures the true positive rate of the top-k recommendation list, i.e., how many target items in the test set are in the top-k list. NDCG further takes the ranking position into consideration, i.e., NDCG is larger if the target items are ranked higher. Both evaluation metrics indicate a better performance if the number is larger.

\subsection{Settings}
We follow the model-specific hyper-parameter settings for each dataset as suggested by the original Bert4Rec implementation \cite{sun2019bert4rec}. For the model extraction attack, we also follow the original implementation and hyper-parameter setting \cite{yue2021black}. We tune the batch size among \{16, 32, 64, 128\}, and tune the $\lambda$ in \autoref{eq:final loss} among \{0.001, 0.01, 0.1, 1.0\}. We assume that the target model can return 100 ranked items for each query sequence, based on which we train the surrogate model. The number of sequences is 3000, which means that the attacker generates 3000 sequences to query the target model. These sequences are generated autoregressively, as suggested in \cite{yue2021black}, where the next item is selected randomly according to the recommendation list returned by the target model.

Previous works for sequential recommendation sample 100 negative items for each test item, and rank the 100 negative items together with the test item to compute the metrics. This may bring sampling bias to the result. We instead rank all the items in the dataset to compute the metrics, which can lead to a more accurate and comprehensive comparison between different methods.

Training a model with GRO is more time-consuming than without it due to the gradient computation. Therefore, we first train the target model until convergence without GRO, then we apply GRO to the converged model and continue to train it for more epochs until convergence. This can significantly speed up the training process. 

\subsection{Results of Identical Model Architectures}
\autoref{fig:main exp} shows the experiment results when both the target model and the surrogate model are Bert4Rec under different defense methods.

For the first two datasets, ML-1M and ML-20M, GRO significantly outperforms the baselines by being able to both preserve the utility of the target model, and decrease the utility of the surrogate model. According to \autoref{fig:ml1m target hr}, \autoref{fig:ml1m target ndcg}, \autoref{fig:ml20m target hr}, \autoref{fig:ml20m target ndcg}, the performance of the target model under GRO is consistently better than Random and Reverse. Then we can see from \autoref{fig:ml1m surrogate hr}, \autoref{fig:ml1m surrogate ndcg}, \autoref{fig:ml20m surrogate hr}, \autoref{fig:ml20m surrogate ndcg}, the attacker's surrogate model performs the worst against the target model that is protected by GRO.

For Steam, the observation is a little bit different, but still demonstrates the superiority of GRO. In \autoref{fig:steam target hr} and \autoref{fig:steam target ndcg}, GRO seems slightly worse than Random and Reverse in preserving the utility of the target model. However, in \autoref{fig:steam surrogate hr} and \autoref{fig:steam surrogate ndcg}, Random and Reverse fail to defend against the model extraction attack. The surrogate model is comparable or even outperforms the target model under Random and Reverse. Instead, GRO can significantly reduce the utility of the surrogate model. Therefore, for Steam, GRO is still successful in defending against the model extraction attack with acceptable sacrifice of the target model's utility.

\subsection{Results of Different Model Architectures}
\autoref{fig:different models} shows the results when the target model is Bert4Rec and the surrogate model is NARM. To be clear, the student model is also Bert4Rec, which is assumed to be consistent with the target model. We show the results on ML-1M and Steam, while ML-20M has similar results. We can observe that GRO is still the best defense method in this case. For ML-1M, GRO not only preserves the utility of the target model, but also significantly reduces the utility of the surrogate model. For Steam, GRO is comparable or slightly worse in preserving the utility of the target model, but it can effectively fool the surrogate model.

\subsection{Analysis of the Choice of $\lambda$}
We further show how different choices of $\lambda$ influence the performance of GRO. \autoref{fig:lambda} shows the experiment results for ML-20M and Steam under different $\lambda$ values, where both the target model and the surrogate model are Bert4Rec. For ML-20M, the best $\lambda$ is 1.0. While for Steam, the best $\lambda$ is 0.001. Besides, the best $\lambda$ for ML-1M is 0.1. We can see that the performance of GRO has no obvious correlation with $\lambda$. Different datasets require different $\lambda$ values in order for GRO to reach the best performance. However, we observe that the best $\lambda$ seems to be correlated with the magnitude of the swap loss. For example, after convergence, the swap loss of ML-20M is close to 1, while the best $\lambda$ is 1; the swap loss of ML-1M is close to 0.1, while the best $\lambda$ is 0.1; the swap loss of Steam is close to 0.001, while the best $\lambda$ is 0.001. However, this needs further investigations and justifications. We leave this for future work.

\subsection{Results of Different Numbers of Sequences}
We show how the number of sequences influence the performance of GRO in \autoref{fig:number seqs}. We present the result of HR on ML-1M. The NDCG presents a similar trend. Other datasets also perform similarly. Since experiments in \autoref{fig:main exp} are conducted with 3000 sequences, we here test the performance of 1000 sequences and 5000 sequences. \autoref{fig:1000 target hr} and \autoref{fig:1000 surrogate hr} show the HRs of the target model and the surrogate model with 1000 sequences, while \autoref{fig:5000 target hr} and \autoref{fig:5000 surrogate hr} show those with 5000 sequences. For both experiments, GRO significantly outperforms Random and Reverse in preserving the utility of the target model. As for the surrogate model, GRO is consistently better than Random in fooling the surrogate model with both 1000 sequences and 5000 sequences. As for Reverse, according to \autoref{fig:1000 surrogate hr}, with 1000 sequences, Reverse is comparable with GRO in fooling the surrogate model. However, in \autoref{fig:5000 surrogate hr}, with 5000 sequences, Reverse performs much worse than before, failing to protect the target model. Note that although Reverse works well in \autoref{fig:1000 surrogate hr}, the corresponding target model's utility is much worse in \autoref{fig:1000 target hr}. Together with the experiments of 3000 sequences in \autoref{fig:main exp}, we can draw a conclusion that GRO is still the best defense method even if the number of sequences varies.

\section{Conclusion}\label{sec:con}
We propose Gradient-based Ranking Optimization (GRO), the first defense method against model extraction attacks on recommender systems. We formalize the defense problem as an optimization problem, and convert the non-differentiable top-k rankings into differentiable swap matrices. We use a student model to learn to extract the target model. We calculate the gradients of the swap matrix with respect to the student model's loss, then convert them into new swap matrices which can maximize the student model's loss. We use a swap loss to force the target model to learn to produce similar rankings as the new swap matrices. Extensive experiments show the superior performance of GRO in decreasing the surrogate model's utility while preserving the target model's utility.

\section*{Acknowledgment}
This research is supported by the Ministry of Education, Singapore, under its Academic Research Fund (Tier 2 Award MOE-T2EP20221-0013 and Tier 2 Award MOE-T2EP20220-0011). Any opinions, findings and conclusions or recommendations expressed in this material are those of the author(s) and do not reflect the views of the Ministry of Education, Singapore. This work is partially supported by the Australian Research Council under the streams of Future Fellowship (No. FT210100624) and Discovery Project (No. DP190101985).

\balance
\bibliographystyle{ACM-Reference-Format}
\bibliography{ref.bib}
\end{document}